\documentclass{article}       
\usepackage{fullpage}
\usepackage{authblk}
\usepackage{url}  
\usepackage{graphicx}  

\usepackage{amsmath,amssymb}
\usepackage{amsthm}
\usepackage{olo}
\usepackage{algorithm}
\usepackage{algpseudocode}
\usepackage{subfigure}
\usepackage{xspace}
\usepackage{mathtools}
\usepackage{hyperref}
\usepackage{xspace}
\usepackage{verbatim}
\usepackage[usenames,dvipsnames,svgnames,table]{xcolor}
\usepackage{stmaryrd}

\newcommand{\svmperf}{\text{SVM$^{\text{\it perf}}$}\xspace}
\newcommand{\kld}{\text{KLD}\xspace}
\newcommand{\ba}{\text{BA}\xspace}
\newcommand{\bakld}{\text{BAKLD}\xspace}

\newcommand{\tpr}{\text{TPR}\xspace}
\newcommand{\tnr}{\text{TNR}\xspace}

\newcommand{\nninit}{\text{NN-init}\xspace}

\newcommand{\dspade}{\textbf{DUPLE}\xspace}
\newcommand{\dspadens}{\textbf{DUPLE-NS}\xspace}
\newcommand{\damp}{\textbf{DAME}\xspace}
\newcommand{\dnemsis}{\textbf{DENIM}\xspace}
\newcommand{\dnemsisns}{\textbf{DENIMS-NS}\xspace}
\newcommand{\nemsis}{\textbf{NEMSIS}\xspace}
\newcommand{\spade}{\textbf{SPADE}\xspace}

\newcommand{\bench}{\textbf{ANN 0-1}\xspace}

\newcommand{\struct}{\textbf{STRUCT-ANN}\xspace}

\begin{document}

\title{Optimizing Non-decomposable Measures with Deep Networks}

\author[2,1]{Amartya Sanyal\footnote{amartya18x@gmail.com}}
\author[2]{Pawan Kumar\footnote{kpawan@cse.iitk.ac.in}}
\author[2]{Purushottam Kar\footnote{purushot@cse.iitk.ac.in}}
\author[3]{Sanjay Chawla\footnote{schawla@qf.org.qa}}
\author[4]{Fabrizio Sebastiani\footnote{fsebastiani@gmail.com}}
\affil[1]{University of Oxford}
\affil[2]{IIT Kanpur}
\affil[3]{Qatar Computing Research Institute}
\affil[4]{Istituto di Scienza e Tecnologia dell'Informazione}

\date{January 28, 2018}


\maketitle

\begin{abstract}
We present a class of algorithms capable of directly training deep neural networks with respect to large families of task-specific performance measures such as the F-measure and the Kullback-Leibler divergence that are structured and non-decomposable. This presents a departure from standard deep learning techniques that typically use squared or cross-entropy loss functions (that are decomposable) to train neural networks. We demonstrate that directly training with task-specific loss functions yields much faster and more stable convergence across problems and datasets. Our proposed algorithms and implementations have several novel features including (i) convergence to first order stationary points despite optimizing complex objective functions; (ii) use of fewer training samples to achieve a desired level of convergence, (iii) a substantial reduction in training time, and (iv) a seamless integration of our implementation into existing symbolic gradient frameworks. We implement our techniques on a variety of deep architectures including multi-layer perceptrons and recurrent neural networks and show that on a variety of benchmark and real data sets, our algorithms outperform traditional approaches to training deep networks, as well as some recent approaches to task-specific training of neural networks.
\end{abstract}


\section{Introduction}
\label{intro}
As deep learning penetrates more and more application areas, there is a natural demand to adapt deep learning techniques to area and task-specific requirements and constraints. An immediate consequence of this is the expectation to perform well with respect to task-specific performance measures. However, this can be challenging, as these performance measures can be quite complex in their structure and be motivated by legacy, rather than algorithmic convenience. Examples include the F-measure that is popular in retrieval tasks, various ranking performance measures such as area-under-the-ROC-curve, and the Kullback-Leibler divergence that is popular in class-ratio estimation problems.

Optimizing these performance measures across application areas has proved to be challenging even when learning linear models, as is evidenced by the recent surge in progress in optimizing ``non-decomposable'' loss functions for learning linear models, as we review in Section~\ref{relatedwork}. The challenge becomes doubly hard when trying to do so while training neural network architectures such as multi-layer perceptrons and convolutional or recurrent neural networks.

The vast majority of training techniques used for neural network at present consist of using simple per-sample loss functions such as least-squares loss or cross-entropy. While their use has allowed research directions to focus more on developing more evolved network architectures, as well as developing highly optimized implementations of training routines on GPU architectures, we show that this is suboptimal and that a sound effort towards training with task-specific loss functions pays off handsomely.\\

\textbf{Our Contributions}
Our work advances the state-of-the-art in training neural networks on a wide variety of non-decomposable performance measures.
\begin{enumerate}
	\item We show how to train neural networks directly with respect to performance measures that are concave, pseudolinear, or nested concave functions.
	\item Our algorithms are readily adapted to neural architectures such as multi-layered perceptrons and recurrent networks, as well be integrated into popular symbolic gradient frameworks such as Theano, TensorFlow, and PyTorch.
	\item Our methods offer far superior performance than traditional cross-entropy based training routines -- on an F-measure maximization task on a benchmark dataset a9a, our method achieves an F-measure of around 0.68 in less than 10 mini-batch iterations whereas it takes traditional cross-entropy based training more than 80 iterations to reach similar performance levels.
	\item Our methods also outperform recently proposed techniques for training deep networks with ranking performance \cite{song2016}. On a benchmark dataset IJCNN, the technique of Song et al. is only able to offer a min-TPR/TNR performance of around 0.55 whereas our technique is able to reach performance over 0.95 in very few iterations.
	\item We apply our techniques to an end-to-end sentimental analysis quantification network and achieve near perfect quantification scores on a challenge dataset using a substantially less number of training iterations.
	\item We offer formal stabilization guarantees for all our algorithms.
\end{enumerate}

\section{Related Work}
\label{relatedwork}
The recent years have seen much interest, as well as progress, in training directly with task-specific performance measures in the field of classification and ranking. Some notable works include those of \cite{KoyejoNRD14,NarasimhanVA14} that investigate the statistical properties of plug-in classifiers for various non-decomposable objectives including F-measure, and \cite{KarLNCS2016,KarSJK13,NarasimhanA13b,Narasimhan:2015eu} which propose stochastic gradient-style algorithms for optimizing non-decomposable performance measures such as F-measure, KL-divergence, area under the ROC curve (AUC), precision recall curve (AUCPR), recall at fixed precision (R@P), etc.

However, all the works cited above focus only on training linear models. Although this allows for simple algorithms for which the works provide very detailed analyses and theoretical guarantees, the approaches do not directly extend to deep networks. Algorithms for deep learning which directly optimize non-decomposable performance measures are relatively unexplored. This can be attributed to the de-facto use of the backpropagation algorithm for training neural networks which crucially depends on the loss function being decomposable.

We are aware of two significant efforts towards training deep networks with non-decomposable performance measures. Below we discuss both to put our contributions in perspective.
\begin{enumerate}
	\item Song et. al.~\cite{song2016} introduce an algorithm for training neural networks for ranking tasks with the average precision as the performance measure. The most key contribution of \cite{song2016} is a result that shows that for nicely behaved non-decomposable loss functions, the expected gradient of the loss function with respect to the network weights can be expressed in terms of standard decomposable loss functions such as cross-entropy and least squares loss.
	\item Eban et. al.~\cite{Eban17} introduce algorithms for optimizing ranking objectives e.g. area under the precision-recall curve and precision at a fixed recall rate.
\end{enumerate}

Both the works above are focussed on ranking measures whereas our work addresses classification and class-ratio estimation (quantification) measures. The applications of classification are various in machine learning and data analysis. The problem of quantification expects accurate estimation of relative prevalence of class labels (e.g. fraction of positive vs negative reviews) and is useful in social engineering and epidemiology.

The work of \cite{song2016} only considers average precision as the performance measure and does not address performance measures we study such as F-meaure and KL divergence. Moreover, we adapted the method proposed in \cite{song2016} to performance measures we study and our experiments show that our precise primal dual techniques far outperform the method of \cite{song2016}.

Although the work of \cite{Eban17} does consider the F-measure which we also study, they do not report any experimentation with F-measure. A possible reason for this might be that their algorithm requires a constrained optimization problem to be solved that is challenging over deep networks. We, on the other hand, provide very generic methods for solving three classes of performance measures which include a large number of widely used measures e.g. H-mean, G-mean, Jaccard coefficient, Q-measure etc which \cite{Eban17} cannot handle.

Furthermore, neither of \cite{song2016,Eban17} offer any convergence guarantees for their proposed algorithms whereas we do offer stabilization and first order stationarity guarantees for our methods.

As a concluding remark, we note that our methods do adapt techniques that were earlier proposed for training linear models, such as in \cite{NarasimhanKJ2015}. However our work differs from existing works, including \cite{NarasimhanKJ2015}, in a significant manner and constitutes an independent contribution. Previous works, such as \cite{NarasimhanKJ2015} only consider linear models which lead to convex problems. We note later in this paper, that a naive and direct application of existing techniques to deep networks yields poor results. The techniques in \cite{NarasimhanKJ2015} cannot be integrated into modern deep learning frameworks like Theano, TensorFlow, PyTorch in scalable manner. Our techniques show how to do so. Moreover, we also provide formal stationarity guarantees for our algorithms when applied to deep networks that \cite{NarasimhanKJ2015} cannot provide since they crucially assume convexity of their problems.

\section{Problem Setting}
\label{sec:formulation}
For sake of simplicity, we restrict ourselves to binary classification problems. Let $\cX \subset \bR^d$ be the space of feature vectors and $\cY = \bc{-1,+1}$ be the label set. The training data set $S$ shall be sampled i.i.d. from some fixed but unknown distribution $\cD$ over $\cX \times \cY$. The proportion of positives in the population and sample $S$ will be denoted by $p = \Pp{(\vx,y)\sim\cD}{y = +1}$ and $\hat p_S$ respectively.

In sharp contrast to most previous work in multivariate optimization that considers only linear models, we concentrate on \emph{non-linear models}, especially those induced by deep neural networks. We will assume that the neural architecture (number of layers, nodes, activation functions and connectivity) has been fixed and let $\cW$ denote the space of all models (weights on the network edges).

To perform learning, we will use a neural model, whose edge weights are indexed by $\vw \in \cW$, to assign a \emph{score} to every data point $\vx \in \cX$ (that can be converted into labels, class probability estimates etc). Linear models typically assign a score by simply computing $\ip{\vw}{\vx}$. However, we will use a more general notation $f(\vx;\vw)$ to denote the score given to the data point $\vx$ by the neural model indexed by the weights $\vw$. The function $f$ can be seen as encoding all the neural connections and activations. We stress that the function $f$ is, in general, neither convex nor concave. We note that this lack of structure in the scoring function precludes a large body of work in linear multivariate optimization and quantification from being applied to deep models.

We will consider performance measures that can be expressed in terms of the true positive rate (TPR) and true negative rate (TNR) of the model. Since TPR and TNR are count-based measures, they are unsuitable for numerical optimization algorithms. For this reason, we consider the use of \emph{reward functions} as surrogates of the TPR and TNR values. A reward function $r$ assigns a \emph{reward} $r(\hat y,y)$ when the true label is $y \in \cY$ but the prediction is $\hat y \in \bR$. Given a reward function $r$, a model $\vw \in \cW$, data point $(\vx,y) \in \cX \times \cY$, and scoring function $f$, we will use
\begin{align*}
r^+(\vw; \vx, y) &= \frac{1}{p}\cdot r(f(\vx;\vw), y)\cdot\ind{y = 1}\\
r^-(\vw; \vx, y) &= \frac{1}{1-p}\cdot r(f(\vx;\vw), y)\cdot\ind{y=-1}
\end{align*}
\noindent to calculate rewards on positive and negative points ($\ind{\cdot}$ denotes the indicator function). The expected value of these rewards will be treated as surrogates of TPR and TNR. Note that since $\E{r^+(\vw; \,\vx, y)}$ = $\E{r(f(\vx;\vw),y)|y = 1}$, setting $r^{0\text{-}1}(\hat{y}, y)$ = $\ind{y\cdot\hat{y} > 0}$ i.e. classification accuracy as the reward function yields $\E{r^+(\vw; \,\vx, y)} = \text{TPR}(\vw)$. We will use the shorthand $P(\vw)=\E{r^+(\vw; \vx, y)}$ to denote population averages of the reward function and, given a sample of $n$ data points $S = \bc{(\vx_1,y_1),\ldots,(\vx_n,y_n)}$, denote the sample average as $\hat P_S(\vw) = \frac{1}{n}\sum_{i=1}^n r^+(\vw; \vx_i, y_i)$ and similarly define $N(\vw), \hat N_S(\vw)$. Unlike previous work \cite{KarLNCS2016,NarasimhanKJ2015}, we will \emph{not} restrict ourselves to concave surrogate reward functions. In particular we will utilize the \emph{sigmoidal} reward, which is widely used as an activation function in neural networks is non-concave: $r_{\text{sigmoid}}(\hat y, y) = (1+\exp(-y\cdot\hat y))^{-1}$

\subsection{Performance Measures}
\label{sec:perf}


\begin{table}[t]
\caption{List of performance measures $\Psi(P,N)$ where $p,n$ denote the TPR and TNR values obtained by the model.}
\centering
{\small\begin{tabular}{cccc}
\hline
Name & Type & Expression $(P,N)$\\
\hline
Min \cite{Vincent94} & Concave & $\min\{P, N\}$\\
Q-Mean \cite{KennedyND09} & Concave & $1 - \sqrt{\frac{(1-P)^2+(1-N)^2}{2}}$\\
F$_\beta$ \cite{Manning+08} & Pseudolinear & $\frac{(1 + \beta^2)\cdot P}{\beta^2 + n/p + P - n/p\cdot N}$\\
KLD \cite{Barranquero:2015fr} & Nested Concave & see text\\
\hline
\end{tabular}}
\label{tab:perf-list}
\end{table}



We will consider three general classes of performance measures, namely, (i) \textit{Concave Performance Measures}, (ii) \textit{Pseudo-linear Performance Measures} and (iii) \textit{Nested Concave Performance Measures}.
In our experiments, we present results on a selection of these performance measures which are listed in Table~\ref{tab:perf-list}.\\

\noindent\textbf{Concave Performance Measures}: These measures can be written as a concave function of the TPR and TNR values:
\[
\cP_\Psi(\vw) = \Psi\br{\tpr(\vw),\tnr(\vw)}
\]
for some concave link function $\Psi: \bR^2 \rightarrow \bR$. These measures are frequently used for cost-sensitive classification in cases with severe label imbalance, for example detection theory \cite{Vincent94}. A popularly used member of this family is the so-called Min-function assigns the value $\min\bc{\tpr(\vw),\tnr(\vw)}$ to a model $\vw$. Note that this compels the model to pay equal attention to both classes. Other examples include the Q-mean and H-mean measures.\\

\noindent\textbf{Pseudo-linear Performance Measures}: These measures can be written as a ratio of two linear functions of the TPR and TNR values of the model, i.e. they have a fractional linear link function. More specifically, given given coefficients $\va,\vb\in\bR^3$,
\[
\cP_{(\va,\vb)}(\vw) = \frac{a_0 + a_1\cdot \tpr(\vw) + a_2\cdot \tnr(\vw)}{b_0 + b_1\cdot \tpr(\vw) + b_2\cdot \tnr(\vw)}.
\]
The popularly used F-measure \cite{Manning+08} is actually a pseudo-linear performance measure in terms of the TPR, TNR values of a model although it is more commonly represented as the harmonic mean of precision and recall. Other members include the Jaccard coefficient and the Gower-Legendre measure.\\

\noindent\textbf{Nested Concave Performance Measures}: Recent works e.g. \cite{Barranquero:2015fr,KarLNCS2016} in problem areas such as quantification and class ratio estimation problems, have brought focus on performance measures that can be written as concave combinations of concave performance measures. More formally, given three concave functions $\Psi, \zeta_1, \zeta_2: \bR^2 \rightarrow \bR$, we define a performance measure
\[
\cP_{(\Psi, \zeta_1, \zeta_2)}(\vw) = \Psi(\zeta_1(\vw),\zeta_2(\vw)),
\]
where $\zeta_i(\vw) := \zeta_i(\tpr(\vw),\tnr(\vw)), i = 1,2$.
%
A widely used measure for quantification tasks is the \textit{\kld: Kullback-Leibler Divergence}~\cite{Barranquero:2015fr,Esuli:2015gh,Gao:2015ly} which can be shown to be a sum of concave functions of the TPR and TNR. If $\vp \in \bR^2$ is the vector of true class priors for a binary classification task and $\hat\vp$ an estimate thereof, then
\begin{align}
  \label{eq:KLD2}
 \kld(\vp,\hat{\vp}) & = \sum_{y\in \mathcal{Y}}p(y)\log\frac{p(y)}{\hat{p}(y)}
\end{align}
$\kld(p,\hat{p}) = 0$ indicates perfect quantification.

We note that there are several other performance measures that our techniques can handle but which we do not discuss here due to lack of space. These include measures for class-imbalanced classification such as H-mean, G-mean, Jaccard coefficient (see \cite{NarasimhanKJ2015}), as well as quantification measures such as Q-measure, NSS and CQB (see \cite{KarLNCS2016}).


\section{Deep Optimization Algorithms}
\label{sec:method}
The task of training deep models directly for quantification performance measures requires us to address the problem of optimizing the concave, nested concave, and pseudolinear performance measures we discussed in Section~\ref{sec:perf} which is challenging due to several reasons: 1) these measures are non-decomposable and do not lend themselves to straightforward training methods such as gradient descent or backpropagation, 2) deep models offer no convenience of convexity, and 3) existing methods for optimizing such measures e.g. \cite{KarLNCS2016,Narasimhan:2015eu} fail to apply directly to deep models. In fact, we will see in Section~\ref{sec:exps} that direct application of traditional techniques yields poor results.

%

%

This section will show how to overcome these challenges to arrive at scalable methods for training deep networks directly on complex non-decomposable measures. A very desirable trait of our methods is that they all enjoy local convergence guarantees. Our techniques also offer far superior empirical performance as compared to typical training methods for deep models.

\begin{algorithm}[t]
\caption{\small \dspade: Dual UPdates for Learning dEep-models}
\label{algo:dspade}
\begin{algorithmic}[1]
	\Require Primal step sizes $\eta_t$, network configuration $\bc{d_\text{in},\text{conf}}$, batch size $b$
	\State $\vw^0 \leftarrow \nninit(d_\text{in},1,\text{conf})$
	\State $\bc{\alpha^0, \beta^0, r_+, r_-, n_+, n_-} \leftarrow 0$
	\For{$t = 1, 2, \ldots, T$}
		\State $S_t \leftarrow$ SAMPLE mini-batch of $b$ data points $\bc{(\vx^t_i,y^t_i)}_{i=1,\ldots,b}$
		\State $\vw^t \leftarrow \vw^{t-1} + \eta_t\cdot\nabla_\vw g(\vw^t; S_t, \alpha^{t-1}, \beta^{t-1})$ \Comment{Primal Step}
		\State $r_+ \leftarrow r_+ + \frac{1}{b}\sum_{i=1}^br^+(\vw^t;\vx^t_i,y^t_i)$ \Comment{Tot. reward on +ves}
		\State $r_- \leftarrow r_- + \frac{1}{b}\sum_{i=1}^br^-(\vw^t;\vx^t_i,y^t_i)$ \Comment{Tot. reward on -ves}
		\State $n_+ \leftarrow n_+ + \frac{1}{b}\sum_{i=1}^n\ind{y^t_i = +1}$ \Comment{Total \# positives}
		\State $n_- \leftarrow n_- + \frac{1}{b}\sum_{i=1}^n\ind{y^t_i = -1}$ \Comment{Total \# negatives}
		\State $(\alpha^{t},\beta^{t}) \leftarrow \displaystyle\argmin_{(\alpha,\beta)}\bs{\alpha\frac{r_+}{n_+} + \beta\frac{r_-}{n_-} - \Psi^\ast(\alpha,\beta)} $\Comment{Dual Step}
	\EndFor
	\State \Return $\vw^T$
\end{algorithmic}
\end{algorithm}

In the following, the procedure $\nninit(d_\text{in},d_\text{out},\text{conf})$ initializes a neural network with $d_\text{in}$ input nodes, $d_\text{out}$ output nodes, and internal configuration (hidden layers, number of internal nodes, connectivity) specified by $\text{conf}$.

\subsection{\dspade: A Deep Learning Technique for Concave Performance Measures}
We present \dspade (Algorithm~\ref{algo:dspade}), a highly scalable stochastic mini-batch primal dual algorithm for training deep models with concave performance measures. 
We shall find it convenient to define the (concave) Fenchel conjugate of the link functions for our performance measures. For any concave function $\Psi$ and $\alpha, \beta \in \bR$, define
\begin{equation}
\Psi^*(\alpha, \beta) = \inf_{u, v \in \bR} \bc{\alpha u + \beta v  - \Psi(u, v)}.
\label{eq:111}
\end{equation}
By the concavity of $\Psi$, we have, for any $u, v \in \bR$,
\begin{equation}
\Psi(u, v) = \inf_{\alpha, \beta \in \bR} \bc{\alpha u + \beta v  - \Psi^*(\alpha, \beta)}.
\label{eq:222}
\end{equation}
The motivation for \dspade comes from a realization that by an application of the Danskin's theorem, a gradient with respect to the $\Psi$ function may be found out by obtaining the maximizer $\alpha,\beta$ values in \eqref{eq:222}. However since this may be expensive, it is much cheaper to update these ``dual'' variables using gradient descent techniques instead. This results in the \dspade algorithm, a primal dual stochastic-gradient based technique that maintains a primal model $\vw \in \cW$ and two dual variables $\alpha,\beta \in \cR$ and alternately updates both using stochastic gradient steps. At every time step it uses its current estimates of the dual variables to update the model, then fix the model and update the dual variables.

We note that \dspade draws upon the \spade algorithm proposed in \cite{NarasimhanKJ2015}. However, its application to deep models requires non-trivial extensions.
\begin{enumerate}
	\item \spade enjoys the fact that gradient updates are very rapid with linear models and is carefree in performing updates on individual data points. Doing so with neural models is too expensive.
	\item Deep model training frameworks are highly optimized to compute gradients over deep networks, especially on GPU platforms. However, they assume that the objective function with respect to which they compute these gradients is static across iterations. \spade violates this principle since it can be seen as taking gradients with respect to a different cost-weighted classification problem at every iteration.
	\item The theoretical convergence guarantees offered by \spade assume that the reward surrogate functions being used are concave functions with respect to the model. As noted in Section~\ref{sec:formulation}, for neural models, even the scoring function $f(\vx;\vw)$ is not a concave/convex function of $\vw$.
\end{enumerate}


\dspade addresses all of the above issues and makes crucial design changes that make it highly optimized for use with deep networks. 

\begin{enumerate}
	\item \dspade overcomes the issue of expensive gradients by amortizing the gradient computation costs over mini-batches. We found this to also improve the stability properties of the algorithm.
	\item To overcome the issue of changing objective functions, \dspade works with an \emph{augmented objective function}. Given a model $\vw \in \cW$, a set $S$ of labeled data points, and scalars $\alpha,\beta$, we define
\[
g(\vw; S, \alpha, \beta) = \alpha\cdot\hat P_S(\vw) + \beta\cdot\hat N_S(\vw).
\]
(see Section~\ref{sec:formulation} for notation). At all time steps, \dspade takes gradients with respect to this augmented objective function instead. We exploit symbolic computation capabilities offered by frameworks such as Theano \cite{Bergstraetal2010} to allow the scalars $\alpha, \beta$ to be updated dynamically and train the network efficiently on a different objective function at each time step.
	\item Our analysis for \dspade makes absolutely no assumptions on the convexity/concavity of the reward and scoring functions. It only requires both functions $r^+,r^-$ to be differentiable almost-everywhere. Thus, \dspade only assumes the bare minimum to allow itself to take gradients.
\end{enumerate}




We are able to show the following convergence guarantee for \dspade (see Appendix~\ref{app:dspade-proof}) 
assuming that the reward functions $r(f(\vx;\vw),y)$ are $L$-smooth functions of the model $\vw$. This is satisfied by all reward functions we consider. Note, however, that nowhere will we assume that the reward functions are concave in the model parameters. We will use the shorthand $\nabla^t = \nabla_\vw g(\vw^t; S_t, \alpha^t, \beta^t)$ and $F(\vw^t,\valpha^t) = g(\vw^t; S_t, \alpha^t, \beta^t)$. Notice that this result assures us that the \dspade procedure will stabilize rapidly and not oscillate indefinitely.

\begin{theorem}
\label{thm:dspade-conv}
Consider a concave performance measure defined using a link function $\Psi$ that is concave and $L'$-smooth. Then, if executed with a uniform step length satisfying $\eta < \frac{2}{L}$, then \dspade $\epsilon$-stabilizes within $\softO{\frac{1}{\epsilon^2}}$ iterations. More specifically, within $T$ iterations, \dspade identifies a model $\vw^t$ such that $\norm{\nabla^t}_2 \leq \bigO{\sqrt{L'\frac{\log T}{T}}}$.
\end{theorem}

\subsection{\dnemsis: Deep Learning with Nested Concave Performance Measures}
We extend the \dspade algorithm to performance measures that involve a nesting of concave functions. To reiterate,
the KLD performance measure which is used extensively for quantification, falls in this category.
These measures are challenging to optimize using \dspade due to their nested structure which prevents a closed form solution for the Fenchel conjugates.

To address this challenge, we present \dnemsis (Algorithm~\ref{algo:dnemsis}) that itself nests its update to parallel the nesting of the performance measures. \dnemsis follows a similar principle as \dspade and is based on the \nemsis algorithm of \cite{KarLNCS2016}. However, the \nemsis algorithm faces the same drawbacks as the \spade algorithm and is unsuitable for training deep models. Due to the more complex nature of the performance measure, \dnemsis works with a slightly different augmented objective function.
\[
h(\vw; S, \valpha, \vbeta, \vgamma) = (\gamma_{1}\alpha_{1} + \gamma_{2}\beta_{1})\cdot\hat P_S(\vw) + (\gamma_{1}\alpha_{2} + \gamma_{2}\beta_{2})\cdot\hat N_S(\vw)
\]
Note that \dnemsis performs inner and outer dual updates that are themselves nested. \dnemsis enjoys similar convergence results as \dspade which we omit for lack of space.

\begin{algorithm}[t]
\caption{\small \dnemsis: A DEep Nested prImal-dual Method}
\label{algo:dnemsis}
\begin{algorithmic}[1]
        \Require Primal step sizes $\eta_t$, network configuration $\bc{d_\text{in},\text{conf}}$, batch size $b$
        \State $\vw^0 \leftarrow \nninit(d_\text{in},1,\text{conf})$
        \State $\bc{\vr^0 , \vq^0, \valpha^0, \vbeta^0, \vgamma^0} \leftarrow (0,0)$
        \For{$t = 1, 2, \ldots, T$}
                \State $S_t \leftarrow$ SAMPLE mini-batch of $b$ data points $\bc{(\vx^t_i,y^t_i)}_{i=1,\ldots,b}$
                \State $\vw^t \leftarrow \vw^{t-1} + \eta_t\cdot\nabla_\vw h(\vw^t; S_t, \valpha^t, \vbeta^t, \vgamma^t)$ \Comment{Primal Step}
                \State $\vq^t \leftarrow (t-1)\cdot\vq^{t-1} + (\alpha^{t-1}_1, \beta^{t-1}_1)\sum_{i=1}^b r^+(\vw^t;\vx^t_i,y^t_i)$
                \State $\vq^t \leftarrow \vq^t + (\alpha^{t-1}_2, \beta^{t-1}_2)\sum_{i=1}^b r^-(\vw^t;\vx^t_i,y^t_i)$
                \State $\vq^t \leftarrow t^{-1}\br{\vq^t - (\zeta_1^\ast(\valpha^t),\zeta_2^\ast(\vbeta^t))}$
                \State $\vr^t \leftarrow t^{-1}\br{(t-1)\cdot\vr^{t-1} + \sum_{i=1}^b(r(\vw^t; \vx^t_i, y^t_i))}$ \footnotemark
                \State $\valpha^t = \underset{\valpha}{\arg\min}\bc{\valpha\cdot\vr^t - \zeta_1^\ast(\valpha)}$ \Comment{Inner Dual Step 1}
                \State $\vbeta^t = \underset{\vbeta}{\arg\min}\bc{\vbeta\cdot\vr^t - \zeta_2^\ast(\vbeta)}$ \Comment{Inner Dual Step 2}
                \State $\vgamma^t = \underset{\vgamma}{\arg\min}\bc{\vgamma\cdot\vq^t - \Psi^\ast(\vgamma)}$ \Comment{Outer Dual Step}
        \EndFor
        \State \Return $\vw^T$
\end{algorithmic}
\end{algorithm}

\footnotetext{$r(\vw^t; \vx^t_i, y^t_i) = (r^+(\vw^t; \vx^t_i, y^t_i), r^-(\vw^t; \vx^t_i, y^t_i)) $}

\subsection{\damp: A Deep Learning Technique for Pseudolinear Performance Measures}
We now present \damp (Algorithm~\ref{algo:damp}), an algorithm to for training deep models on pseudolinear performance measures such as F-measure which are extremely popular in several areas and direct optimization routines are sought after. We recall that although the work of \cite{Eban17} does discuss F-measure optimization, we do not have access to any scalable implementations of the same. Our algorithm \damp, on the other hand, is based on an alternating strategy, is very scalable and gives superior performance across tasks and datasets. For sake of simplicity, we represent the pseudolinear performance measure as
\[
\cP_{(\va,\vb)}(\vw) = \frac{\cP_{\va}(\vw)}{\cP_{\vb}(\vw)} = \frac{a_0 + a_1\cdot \tpr(\vw) + a_2\cdot \tnr(\vw)}{b_0 + b_1\cdot \tpr(\vw) + b_2\cdot \tnr(\vw)}
\]
We now define the notion of a \emph{valuation function}.
\begin{definition}[Valuation Function]
The valuation of a pseudolinear measure $\cP_{(\va,\vb)}(\vw)$ at any level $v > 0$, is defined to be $V(\vw,v) = \cP_{\va}(\vw) - v\cdot\cP_{\vb}(\vw)$
\end{definition}

\begin{algorithm}[t]
	\caption{\small \damp: A Deep Alternating Maximization mEthod}
	\label{algo:damp}
	\begin{algorithmic}[1]
		\Require Training dataset $T = \bc{(\vx_i,y_i)}_{i=1}^n$, step lengths $\eta_t$, network configuration $\bc{d_\text{in},d_\text{int},d_\text{out},\text{conf}_1,\text{conf}_2}$, batch size $b$
		\State $\vw^{-1}_1 \leftarrow \nninit(d_\text{int},1,\text{conf}_1)$
		\State $\vw^{-1}_2 \leftarrow \nninit(d_\text{in},d_\text{int},\text{conf}_2)$
		\State $(\vw^{0,0}_1,\vw^0_2) \leftarrow$ Pre-train on cross-entropy on dataset $T$
		\State Create new dataset $\tilde T = \bc{(f(\vx_i,\vw^0_2),y_i)}_{i=1}^n$ \Comment{New features}
		\For{$t=1,2,\ldots,T$}
		\State $S_{t,0} \leftarrow$ SAMPLE mini-batch of $b$ data points $(\vz^{t,t'}_i,y^{t,t'}_i)$
		\State $v^t \leftarrow \cP_{(\va,\vb),S_{t,0}}(\vw^{t-1}_1,\vw^0_2)$
			\For{$t'=1,2,\ldots,T'$}
				\State $S_{t,t'} \leftarrow$ SAMPLE mini-batch of $b$ data points $(\vz^{t,t'}_i,y^{t,t'}_i)$
				\State $\vw^{t-1,t'}_1 \leftarrow{} \vw^{t-1,t'-1}_1  +\eta_t\cdot\nabla_{\vw^{t-1,t'-1}_1} V_{S_{t,t'}}((\vw^{t-1,t'-1}_1,\vw^0_2),v^t)$
			\EndFor
			\State $\vw^t_1 \leftarrow \vw^{t-1,T'}_1$
		\EndFor
		\State \Return $(\vw^T_1,\vw^0_2)$
	\end{algorithmic}
\end{algorithm}

\damp makes use of two simple observations in its operation: 1) A model $\vw$ has good performance i.e. $\cP_{(\va,\vb)}(\vw) > v$ iff it satisfies $V(\vw,v) > 0$, and 2) the valuation function itself is a performance measure but a decomposable one, corresponding to a cost-weighted binary classification problem with the costs given by the weights $\va,\vb$ and $v$.
 
We will use the notation $\cP_{(\va,\vb),S}(\vw)$ and $V_S(\vw,v)$ to denote respectively, the performance measure, and the valuation function as defined on a data sample $S$. At every time step $t$, \damp looks at $v^t = \cP_{(\va,\vb)}(\vw^t)$ and attempts to approximate the task of optimizing F-measure (or any other pseudolinear measure) using a cost weighted classification problem described by the valuation function at level $v^t$. After updating the model with respect to this approximation, \damp refines the approximation again, and so on.

We note that similar alternating strategies have been studied in literature in the context of F-measure before \cite{KoyejoNRD14,NarasimhanKJ2015} and offer provable convergence guarantees for linear models. However, a direct implementation of these methods gives extremely poor results as we shall see in the next section. The complex nature of these performance measures, that are neither convex nor concave, make it more challenging to train deep models.

To solve this problem, \damp utilizes a two-stage training procedure, involving pretraining the entire network (i.e. both upper and lower layers) on a standard training objective such as cross-entropy or least squares, followed by fine tuning of \emph{only the upper layers} of the network to optimize F-measure. The pretraining is done using standard stochastic mini-batch gradient descent.

For sake of simplicity we will let $(\vw_1,\vw_2)$ denote a stacking of the neural networks described by the models $\vw_1$ and $\vw_2$. More specifically $\vw_2$ denotes a network with input dimensionality $d_\text{in}$ and output dimensionality $d_\text{int}$ whereas $\vw_1$ denotes a network with input dimensionality $d_\text{int}$ and output dimensionality $d_\text{out}$. To ensure differentiability, \damp uses valuation functions with appropriate reward functions replacing the \tpr and \tnr functions.

We are able to show a stronger \emph{first order stationary convergence} guarantee for \damp. For sake of simplicity, we present the proof for the batch version of the algorithm in Appendix~\ref{app:damp}. We only assume that the valuation functions are $L$-smooth functions of the upper model $\vw_1$. It is also noteworthy that we present the guarantee only for the fine-tuning phase since the pre-training phase enjoys local convergence guarantees by standard arguments. For this reason, we will omit the lower network in the analysis. We also assume that the performance measure satisfies $\cP_{\va}(\vw) \leq M$ for all $\vw \in \cW$ and $\cdot\cP_{\vb}(\vw) \geq m$ for all $\vw \in cW$. We note that these assumptions are standard \cite{KarLNCS2016,NarasimhanKJ2015} and also readily satisfied by F-measure, Jaccard coefficient etc for which we have $m, M = \Theta(1)$ (see \cite{NarasimhanKJ2015}). Let $\kappa = 1 + M/m$. Then we have the following result.

\begin{theorem}
\label{thm:damp-conv}
If executed with a uniform step length satisfying $\eta < \frac{2}{L\kappa}$, \damp discovers an $\epsilon$-stable model within $\bigO{\frac{1}{\epsilon^2}}$ inner iterations. More specifically, for $t \leq \frac{\kappa^2}{m}\frac{1}{\eta\br{1 - \frac{L\kappa\eta}{2}}\epsilon^2}$, \damp identifies a model $\vw^t_1$ such that $\norm{\nabla_{\vw}\cP_{(\va,\vb)}(\vw)} \leq \epsilon$.
\end{theorem}

\section{Experimental Results}
\label{sec:exps}


\begin{table}
\centering
    {\small\begin{tabular}{|c|c|c|c|c|}
		
    \hline
    \textbf{Data Set}               &\textbf{\# Points}   & \textbf{Feat.}     & \textbf{Positives} & \textbf{Source}\\\hline
      KDDCup08        & 100K               & 117           & 0.61\% & KDDCup08  \\\hline
    PPI                     & 240K               & \phantom{0}85         & 1.19\% & \cite{ppiQBK06}\\\hline
    CoverType       & 580K               & \phantom{0}54         & 1.63\% & UCI\\\hline
    Letter          & \phantom{0}20K             & \phantom{0}16         & 3.92\% & UCI \\\hline
    IJCNN-1         & 140K               & \phantom{0}22         & 9.57\% & UCI\\\hline
    Adult           & \phantom{0}50K             & 123           & 23.93\% & UCI\\\hline
    Twitter        &10K                 &NA         & 77.4\% & SEMEVAL16 \\ \hline
		
\end{tabular}}

\caption{Statistics of data sets used.}
\label{tab:dataset-stats}
\end{table}

We performed extensive evaluation of \dspade, \dnemsis and \damp on benchmark and real-life challenge datesets and found it to outperform both traditional techniques for training neural networks, as well as the more nuanced task-driven training techniques proposed in the work of \cite{song2016}.

\textbf{Datasets}: We use the datasets listed in Table \ref{tab:dataset-stats}. Twitter refers to the dataset revealed as a part of the SEMEVAL 2016 sentiment detection challenge \cite{SemEval:2016:task4:ISTI-CNR}.

\textbf{Competing Methods}: We implemented and adapated several benchmarks from past literature in an attempt to critically assess the performance of our methods.
\begin{enumerate}
	\item \bench refers to a benchmark multi-layer perceptron model trained using the cross-entropy loss functions to minimize the misclassification rate.
	\item \struct refers to an adaptation of the structured optimization algorithm from \cite{song2016} to various performance measures (implementation details in the Appendix~\ref{sec:structual-ann}).
	\item \textbf{ANN-PG} refers to an implementation of a plug-in classifier for F-measure as suggested in \cite{KoyejoNRD14}.
	\item \dnemsisns refers to a variant of the \nemsis algorithm that uses a count based reward instead of sigmoidal rewards. A similar benchmark was constructed for \dspade as well.  
\end{enumerate}

For lack of space, some experimental results for the \damp algorithm are included in Appendix~\ref{app:damp}. All hyper-parameters including model architecture were kept the same for all algorithms. Learning rates were optimized to give best results.


\subsection{Experiments on Concave Measures}
The results with \dspade (Figures~\ref{fig:MinTPRTNR} and \ref{fig:QMean}) on optimizing the MinTPRTNR and the QMean performance measures, show that \dspade offers very fast convergence in comparison to \bench. It is to be noted that on MinTPRTNR, \bench has a very hard time obtaining a non-trivial score. For the experiment on \textbf{IJCNN1}, we ran the experiment for a longer time to allow \bench and \struct to converge and we observe that they are highly time intensive, when compared to \dspade.

These experiments show that \dspade and its variant \dspadens outperform the competitors both in terms of speed as well as accuracy. It is also to be noted that \dspade not only takes lesser iterations than \struct but each iteration of \dspade is at least 10X faster than that of \struct.

\begin{figure*}[t]
\centering
\subfigure[PPI]{
\includegraphics[width=1.5in]{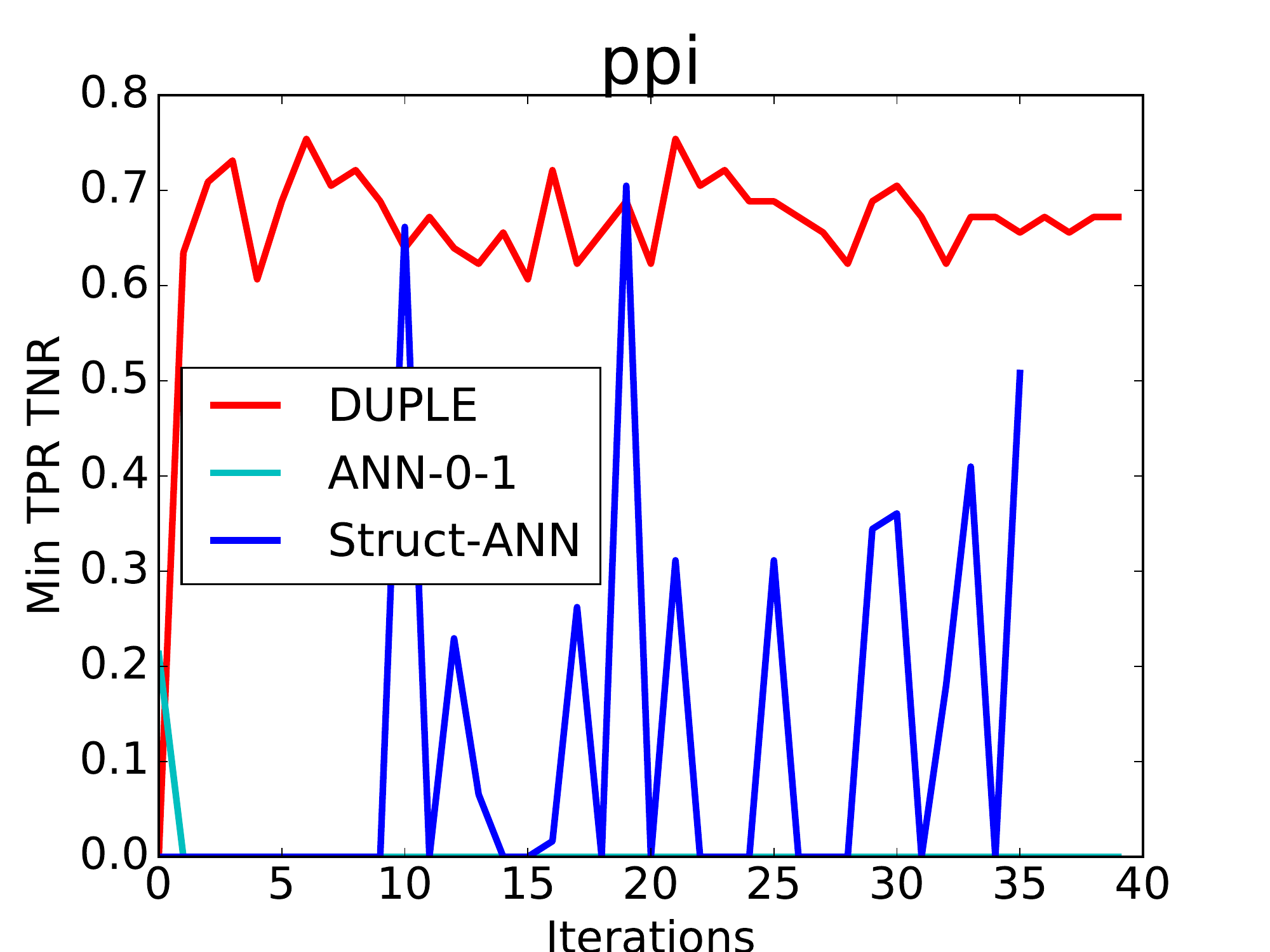}
\label{subfig:kld-ppi}
}\hspace*{-5pt}
\subfigure[KDD08]{
\includegraphics[width=1.5in]{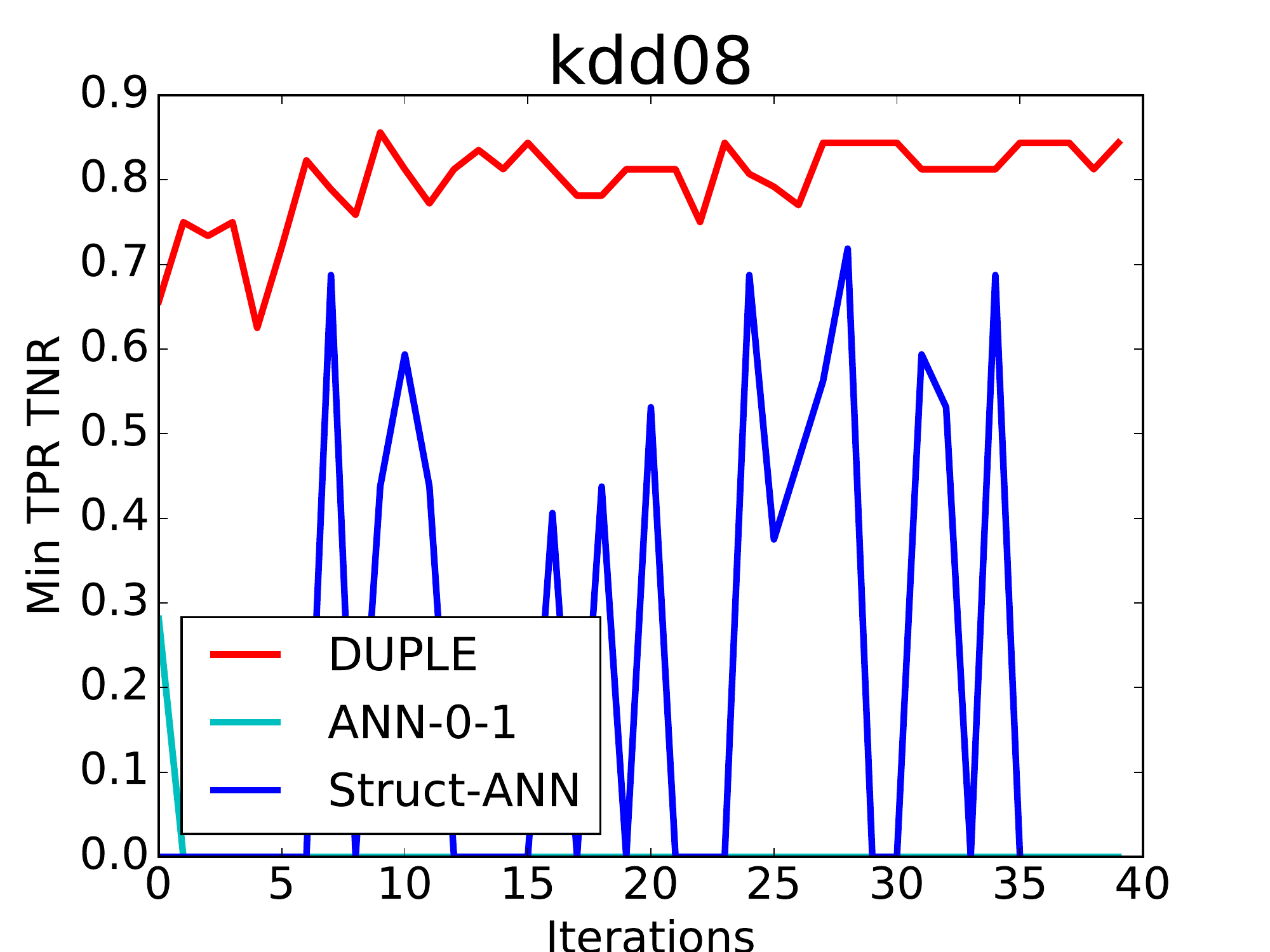}
\label{subfig:kld-letter}
}\hspace*{-5pt}
\subfigure[COVT]{
\includegraphics[width=1.5in]{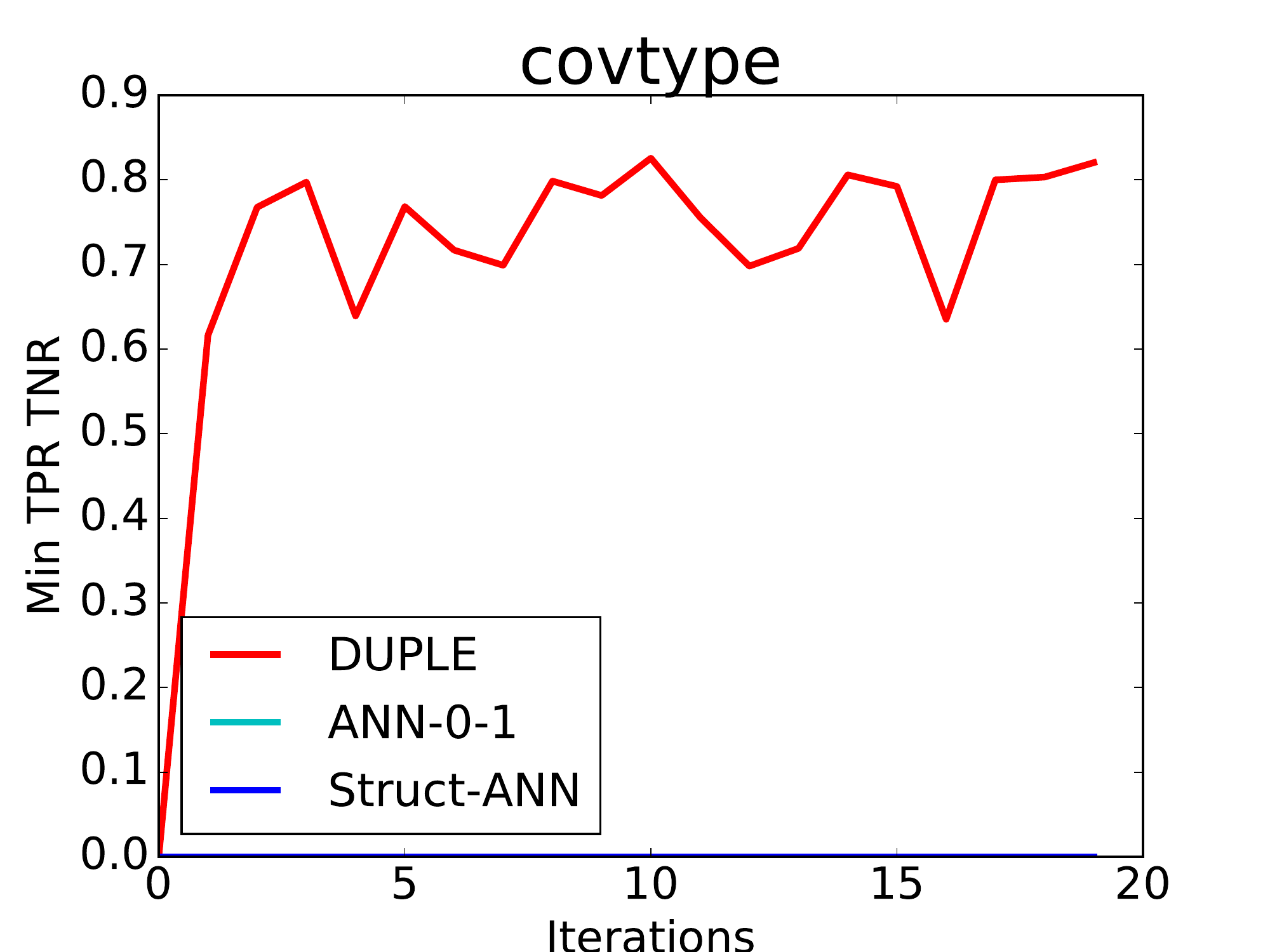}
\label{subfig:kld-covt}
}\hspace*{-5pt}
\subfigure[IJCNN]{
\includegraphics[width=1.5in]{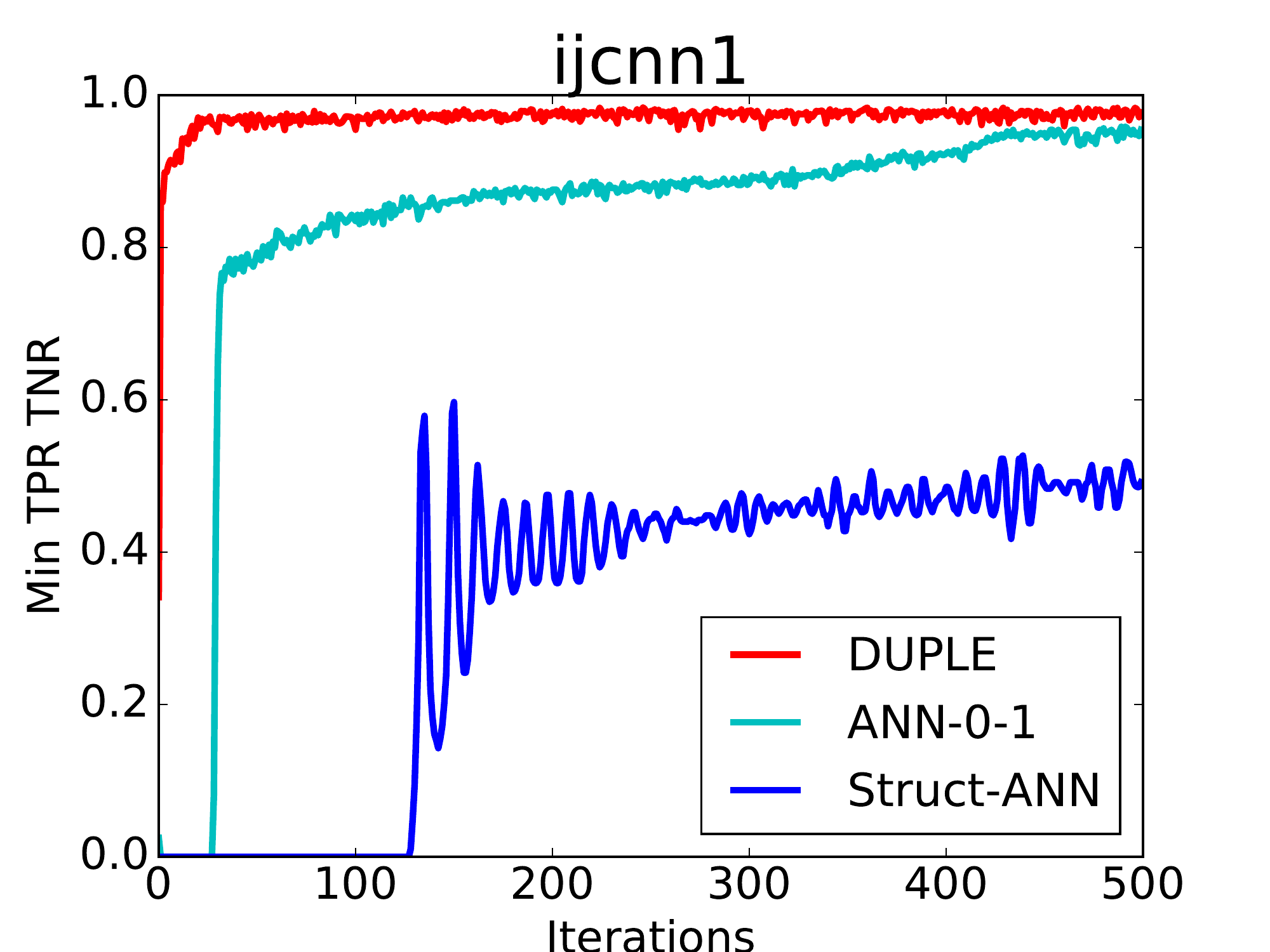}
\label{subfig:kld-ppi}
}

\caption{Experiments on maximizing MinTPRTNR, a concave performance measure}
\label{fig:MinTPRTNR}
\end{figure*}


\begin{figure*}[h!]
\centering
\subfigure[PPI]{
\includegraphics[width=1.5in]{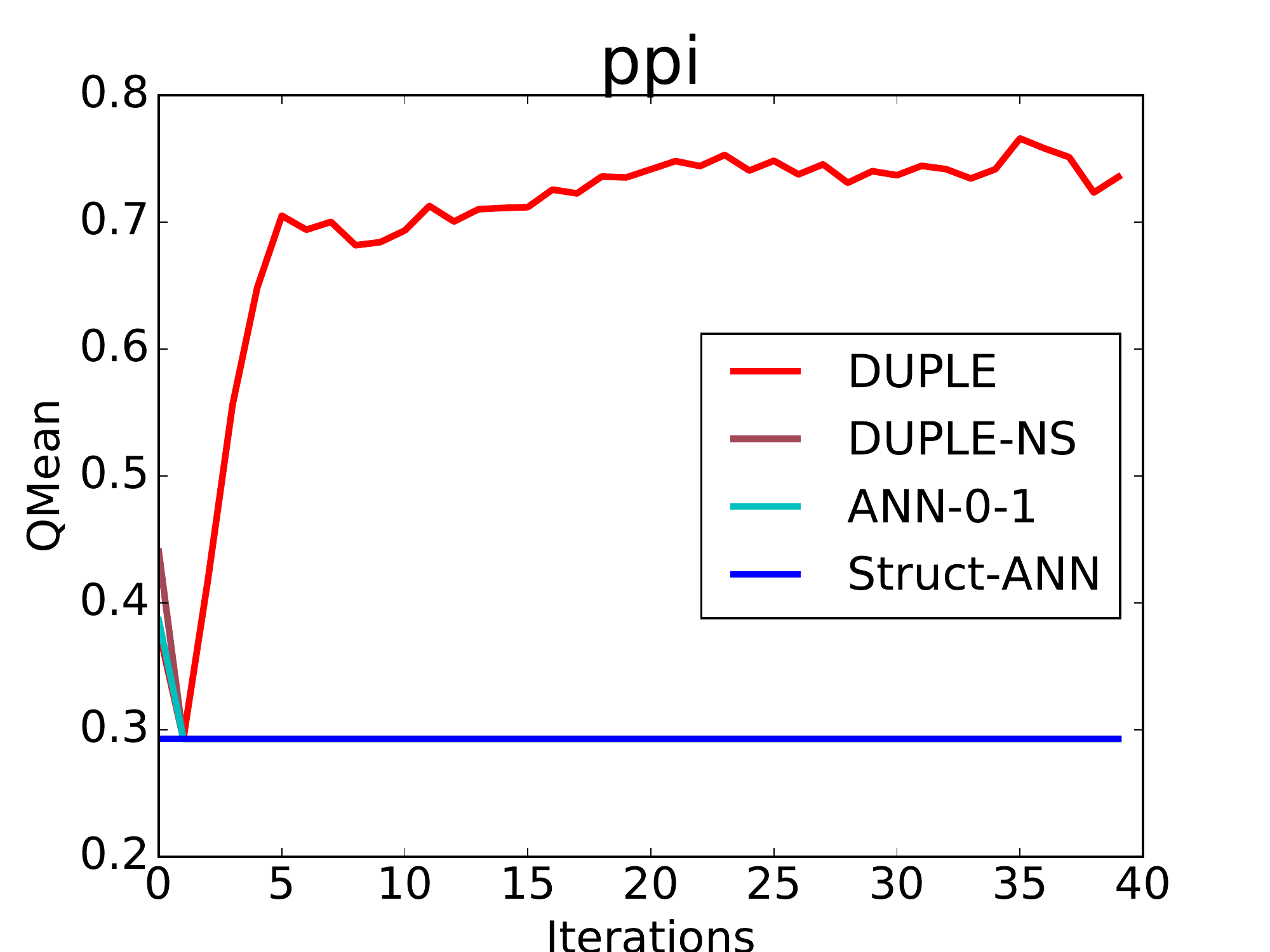}
\label{subfig:kld-ppi}
}\hspace*{-5pt}
\subfigure[KDD08]{
\includegraphics[width=1.5in]{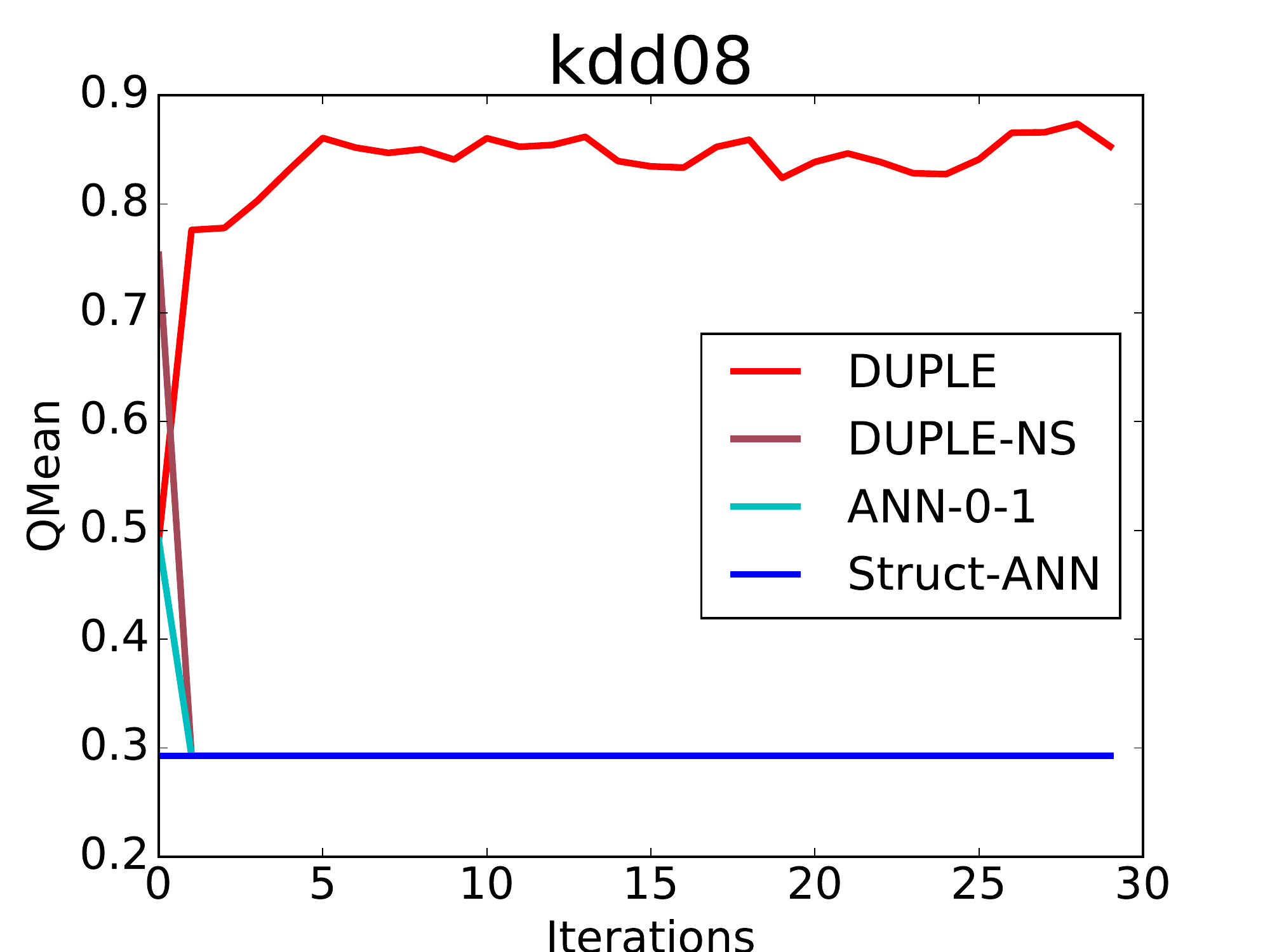}
\label{subfig:kld-kdd08}
}\hspace*{-5pt}
\subfigure[IJCNN1]{
\includegraphics[width=1.5in]{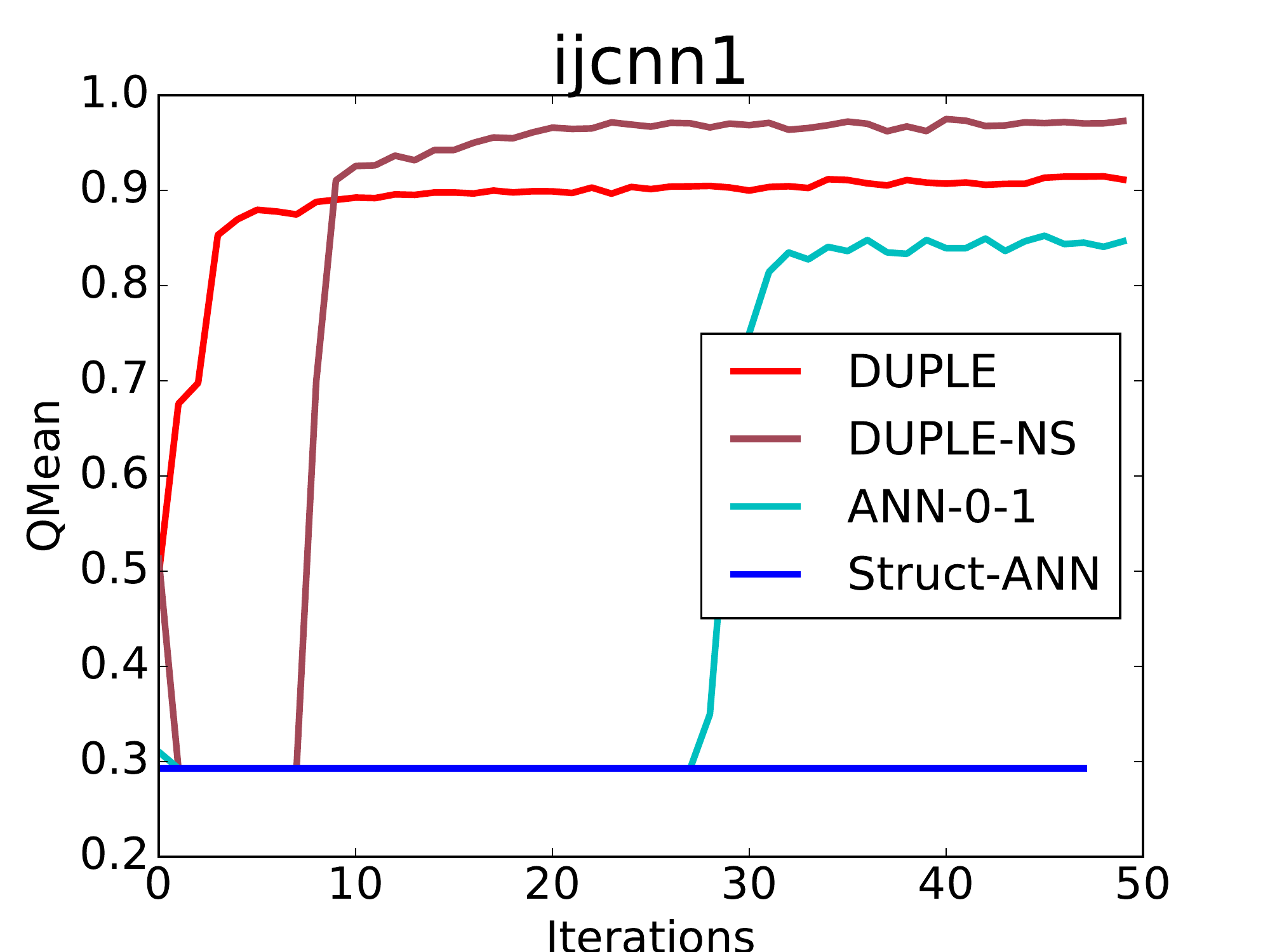}
\label{subfig:kld-covt}
}\hspace*{-5pt}
\subfigure[A9A]{
\includegraphics[width=1.5in]{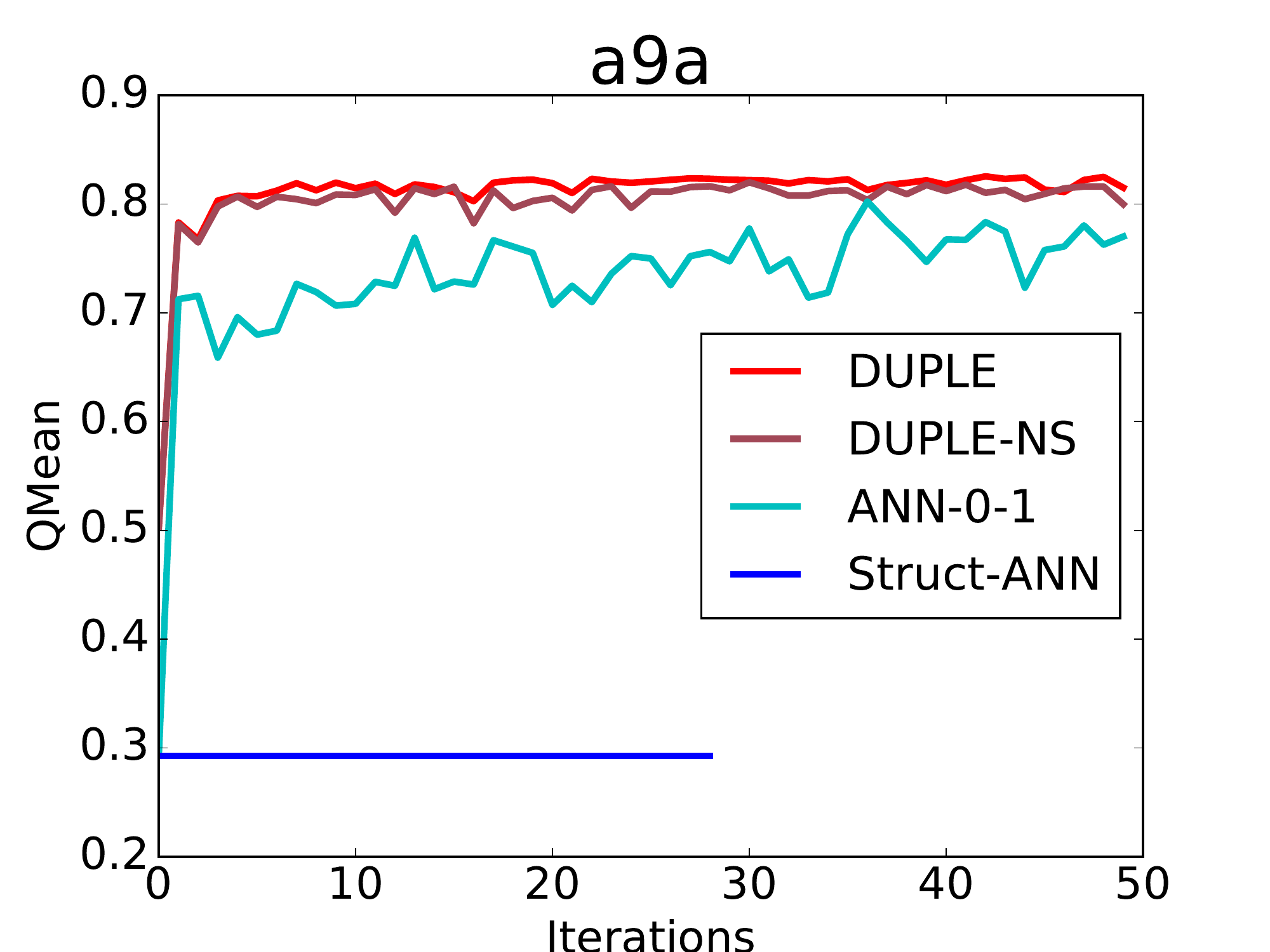}
\label{subfig:kld-ppi}
}\hspace*{-5pt}

\caption{Experiments on maximizing QMean, a concave performance measure}

\label{fig:QMean}
\end{figure*}

\begin{figure*}[h!]
\centering
\subfigure[KDD08]{
\includegraphics[width=1.5in]{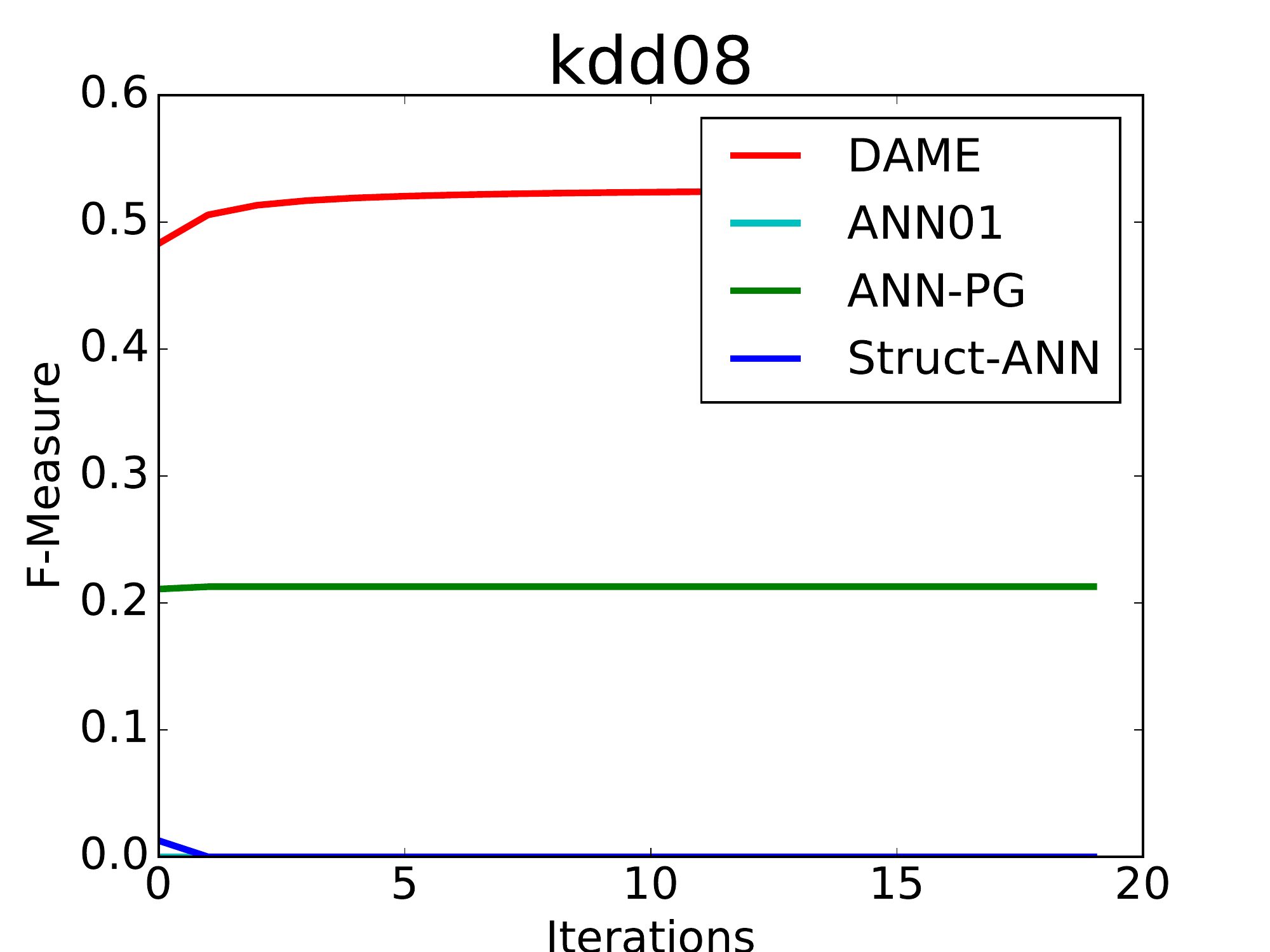}
\label{subfig:f1-kdd}
}\hspace*{-5pt}
\subfigure[COD-RNA]{
\includegraphics[width=1.5in]{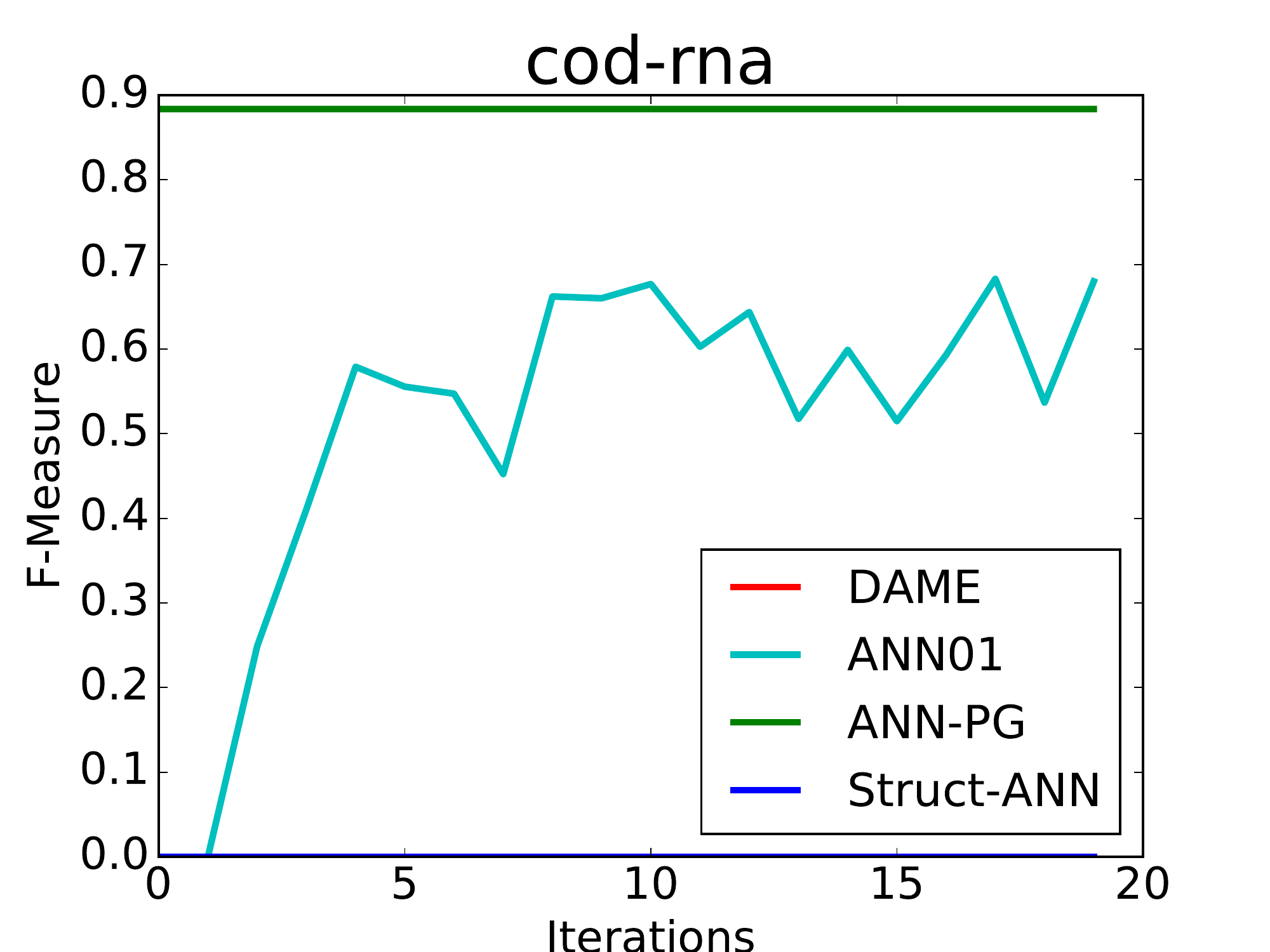}
\label{subfig:f1-cod}
}\hspace*{-5pt}
\subfigure[LETTER]{
\includegraphics[width=1.5in]{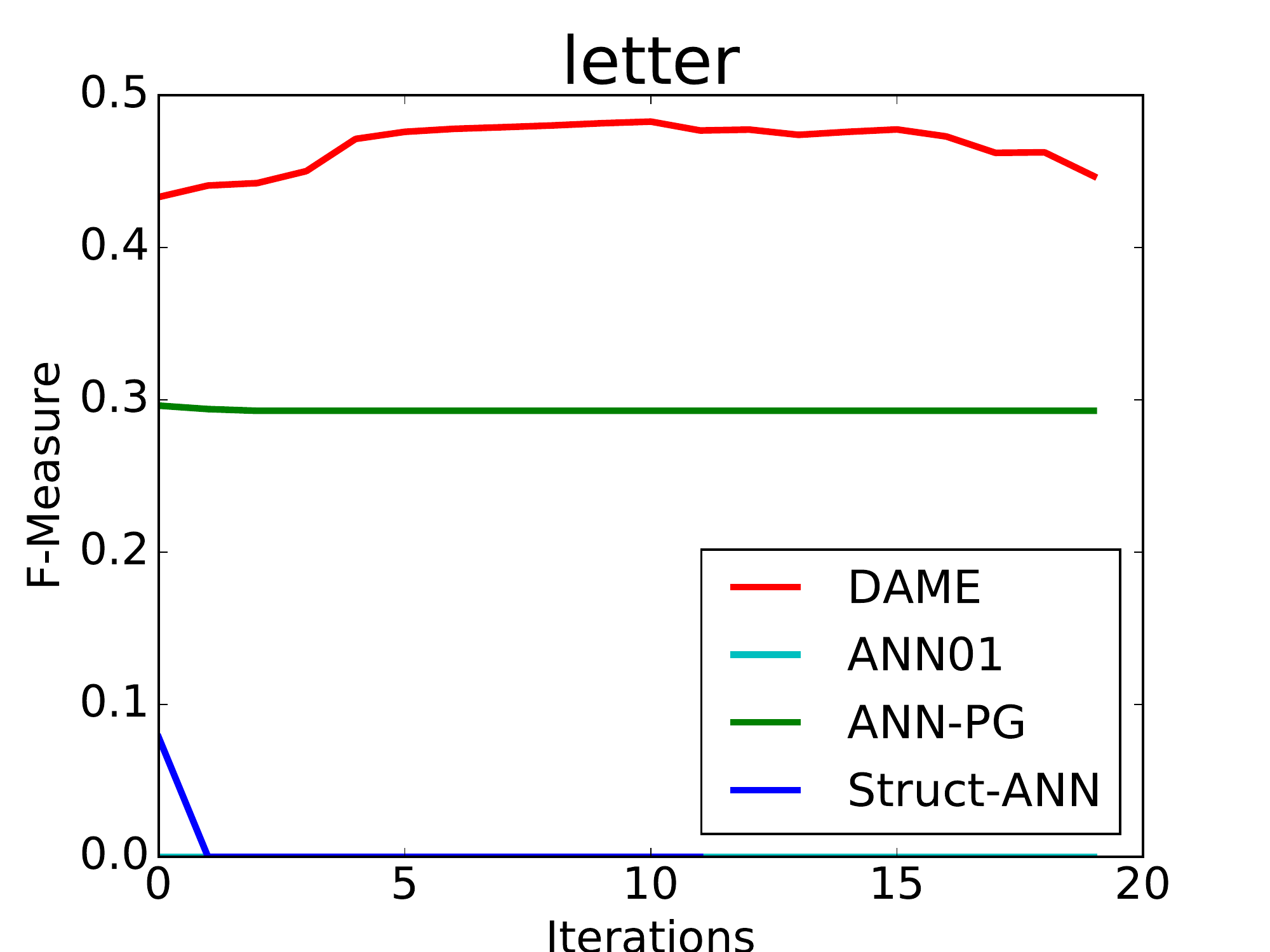}
\label{subfig:f1-letter}
}\hspace*{-5pt}
\subfigure[A9A]{
\includegraphics[width=1.5in]{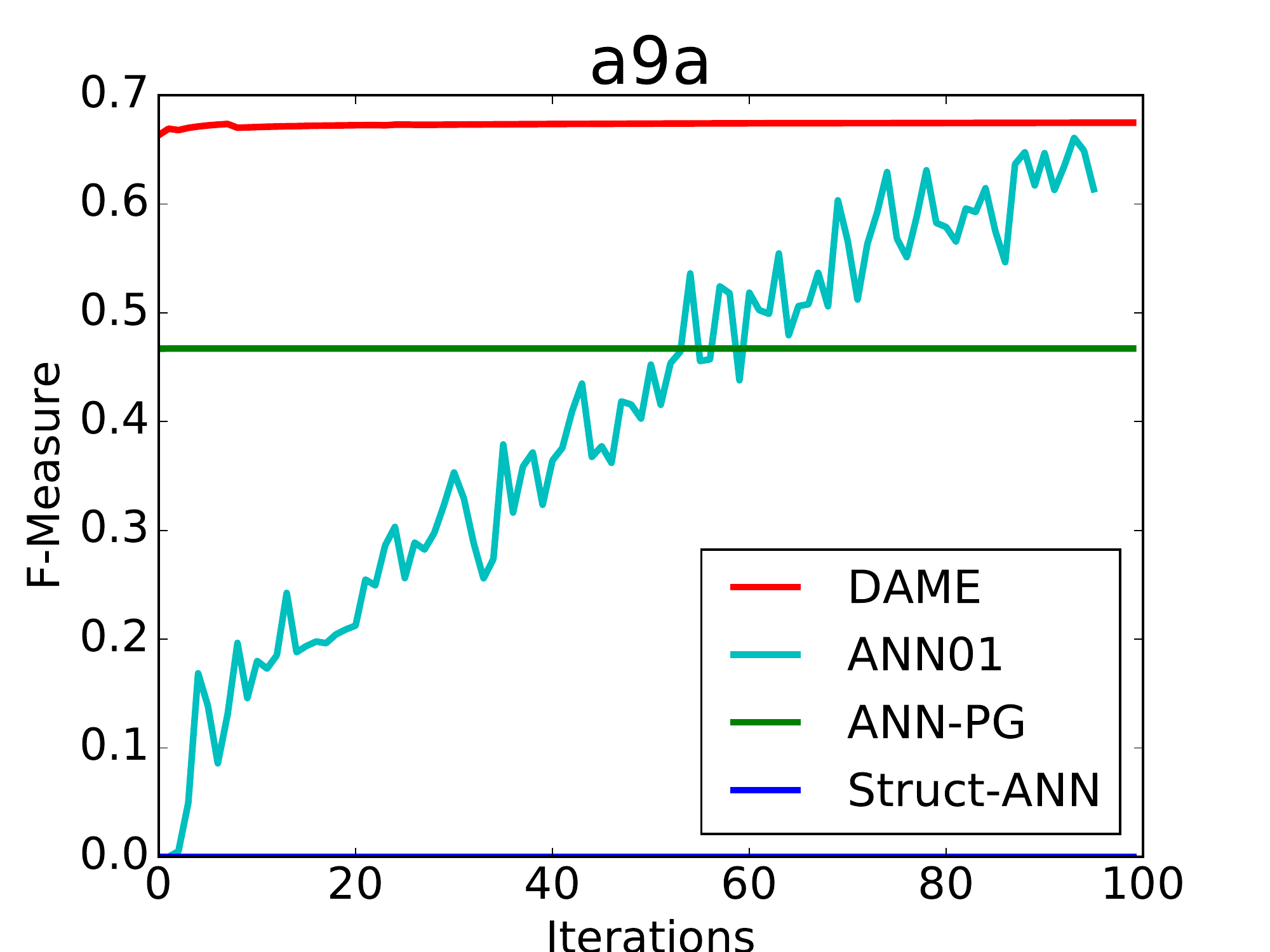}
\label{subfig:f1-a9a}
}
\caption{Experiments on maximizing F-measure, a pseudolinear performance measure}
\label{fig:F1}
\end{figure*}

\subsection{Experiments with Nested Performance Measures}
In Figure~\ref{fig:KLD}, we can see the results obtained by \dnemsis while optimizing the \kld performance measure. It shows rapid convergence to near-perfect quantification (class ratio estimation) scores. The experiments also show that \dnemsis and \dnemsisns require far less iterations than its competitor \bench (whenever \bench is successful at all). The \struct benchmark is not shown for these experiments since it always got a value close to 0 by trivially predicting every data point as negative as the datasets are highly biased.
\begin{figure*}[h!]
\centering
\subfigure[PPI]{
\includegraphics[width=1.5in]{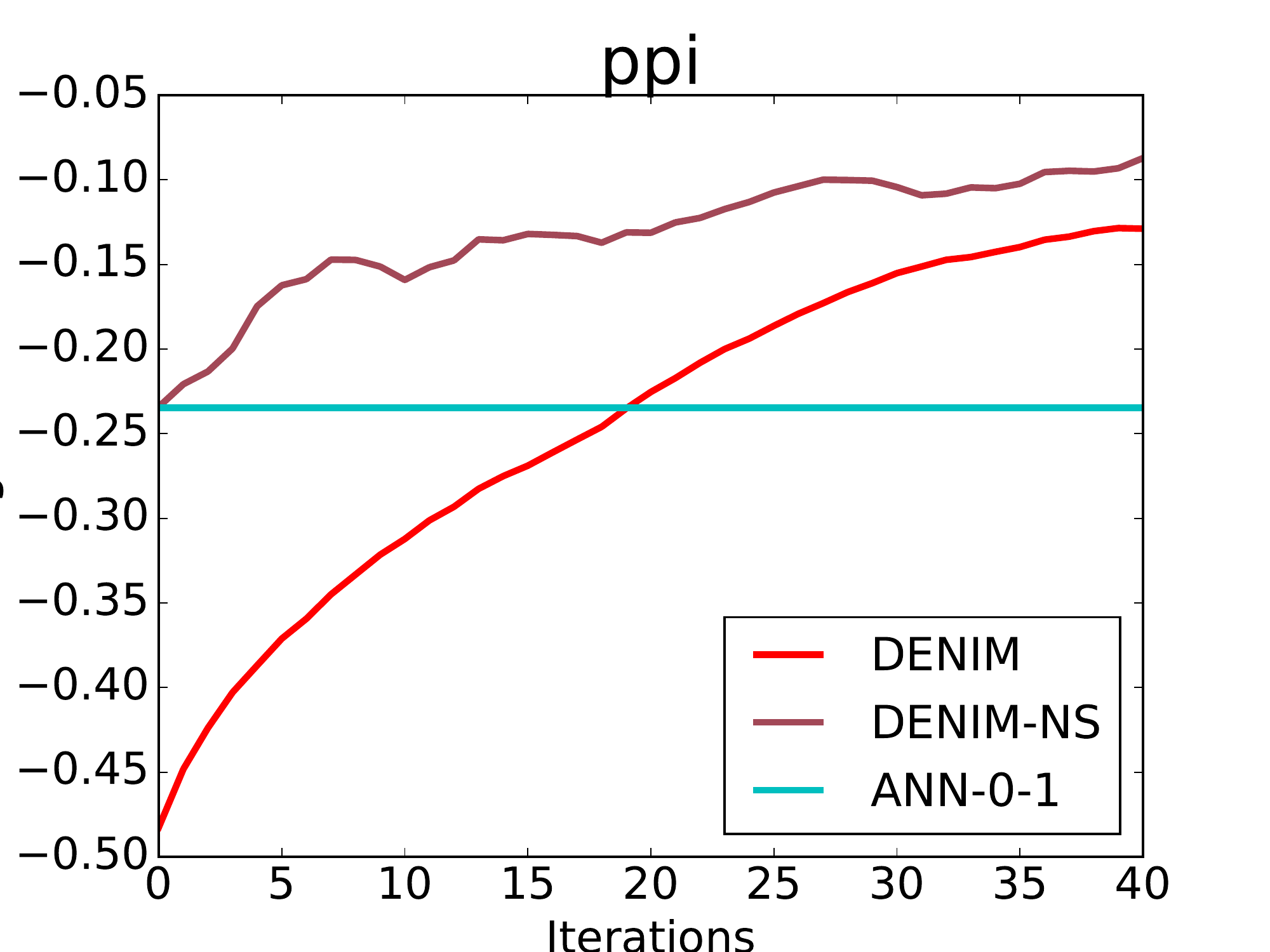}
\label{subfig:kld-ppi}
}\hspace*{-2pt}
\subfigure[Letter]{
\includegraphics[width=1.5in]{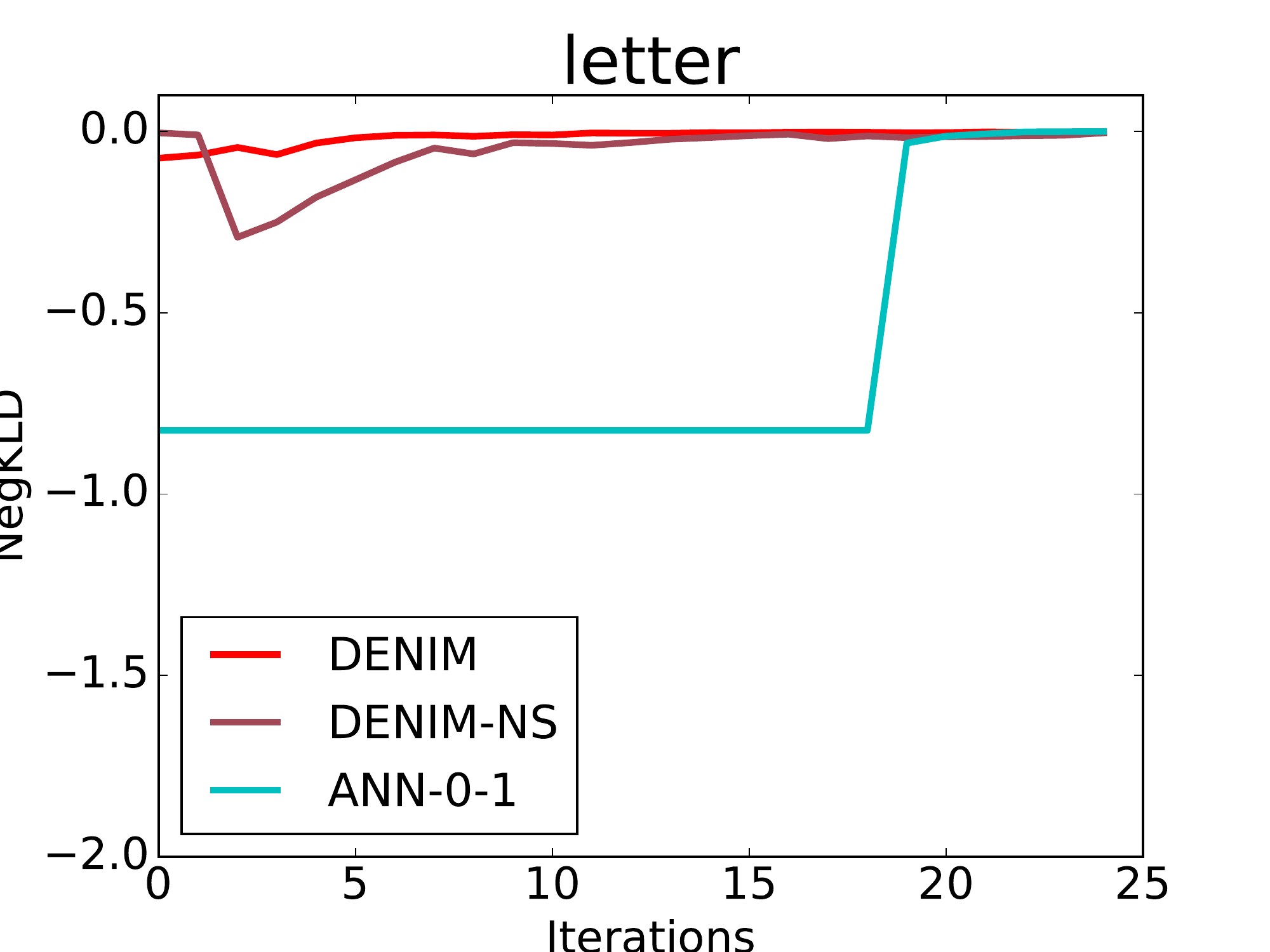}
\label{subfig:kld-ppi}
}\hspace*{-2pt}
\subfigure[COVT]{
\includegraphics[width=1.5in]{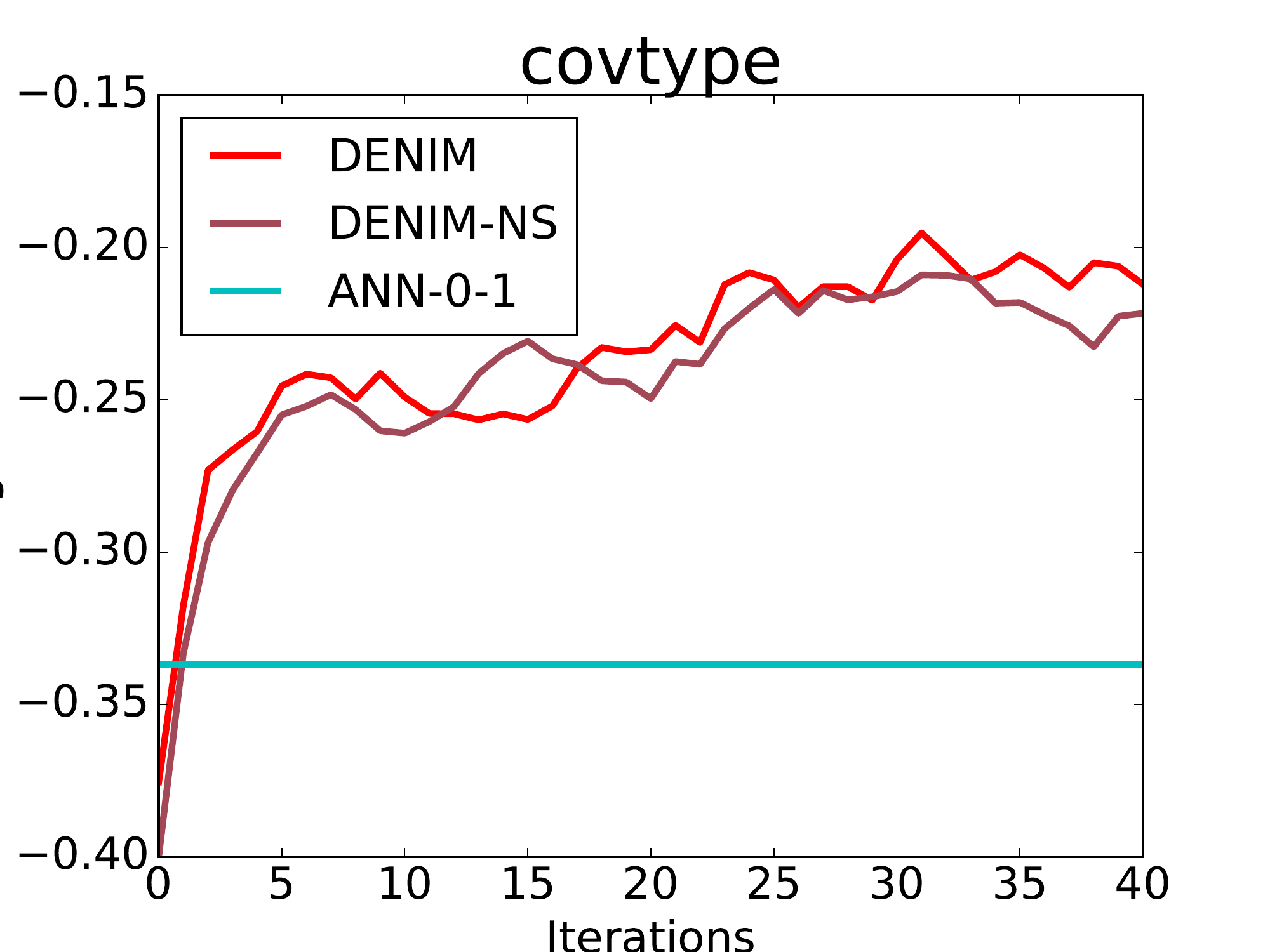}
\label{subfig:kld-a9a}
}\hspace*{-2pt}
\subfigure[IJCNN]{
\includegraphics[width=1.5in]{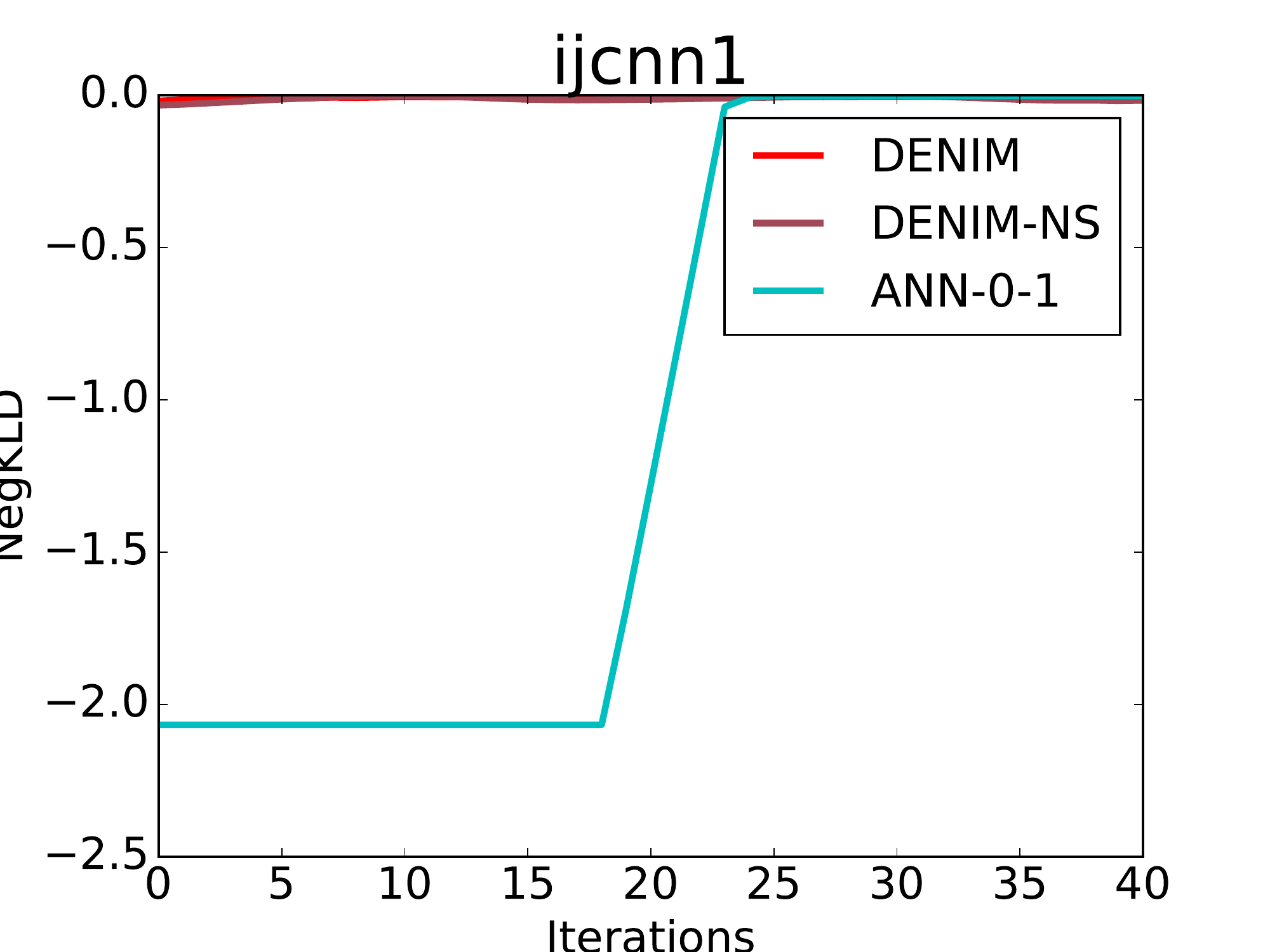}
\label{subfig:kld-ppi}
}\hspace*{-10pt}
\caption{Experiments on minimizing Kullback Leibler divergence, a nested concave performance measure}
\label{fig:KLD}
\end{figure*}

\subsection{Experiments with Pseudolinear Measures}
Figure~\ref{fig:F1} (in the Supplementary material) shows the performance of \damp on optimizing the F1-Measure. A naive training with misclassification loss yields extremely poor F-measure performance. Moreover, plug-in methods such as those proposed in \cite{KoyejoNRD14} linear models also perform very poorly. \damp on the other hand is able to rapidly offer very good F-measure scores after looking at a fraction of the total data.

As, it can be seen, \struct offers a consistently poor performance whereas it seems to perform well in the linear case. This is because our implementation of the \textbf{Struct ANN} is a minibatch method and the gradient obtained from the structual method has almost no real information due to this. In other variants of the application \cite{song2016} of the \struct algorithm, full batch methods were used. We would like to point out that the use of the entire training dataset for every update is extremely expensive with respect to memory and computation time, especially when working with GPU architectures.

\subsection{Case Study: Quantification for Sentiment Analysis}
We report the results of experiments comparing the performance of the \dnemsis on a Twitter sentiment detection challenge problem. The task in this challenge was to ascertain correctly the fraction of tweets exhibiting various sentiments. The performance was measured using the Kullback-Leibler divergence (see \eqref{eq:KLD2}). We trained an end-to-end DeepLSTM model trained using \dnemsis. We also trained an attention-enabled network for the same task using the \dnemsis. Our models accepted raw text in the standard one-hot encoding format and performed task specific optimization and generated task specific vocabulary embeddings. Our representations were 64-dimensional and were learnt jointly with other network parameters.

\noindent\textbf{Implementation details}: All our LSTM models used a single hidden layer with 64 hidden nodes, which gave rise to 64-dimensional hidden state representations. For the LSTM model, the final label was obtained by applying a linear model with a logistic wrapper function. For the attention models (referred to as \textbf{AM}), the decoder hidden states were set to be 64-dimensional as well. The alignment model was set to be a feed-forward model with a softmax layer. Step lengths were tuned using standard implementations of the ADAM method. Training was done by adapting the \dnemsis method.

\dnemsis is able to obtain near perfect quantification on both LSTM (\kld = 0.007) as well as AM (\kld = 0.00002) models (see Figure~\ref{fig:conv}). In contrast, the classical cross-entropy method with attention model (\textbf{AM-CE}) is unable to obtain satisfactory performance. \dnemsis converges to optimal test \kld performance in not only far lesser iterations, but also by using far less data samples. Also note that the \textbf{AM} models trained with \dnemsis give \kld losses that are 2 orders of magnitude smaller than what LSTMs offer when trained with \dnemsis.

\begin{figure}[h!]
  \centering
    \subfigure[Convergence to optimal test \kld performance for different RNN models.]{
\includegraphics[height=1.3in]{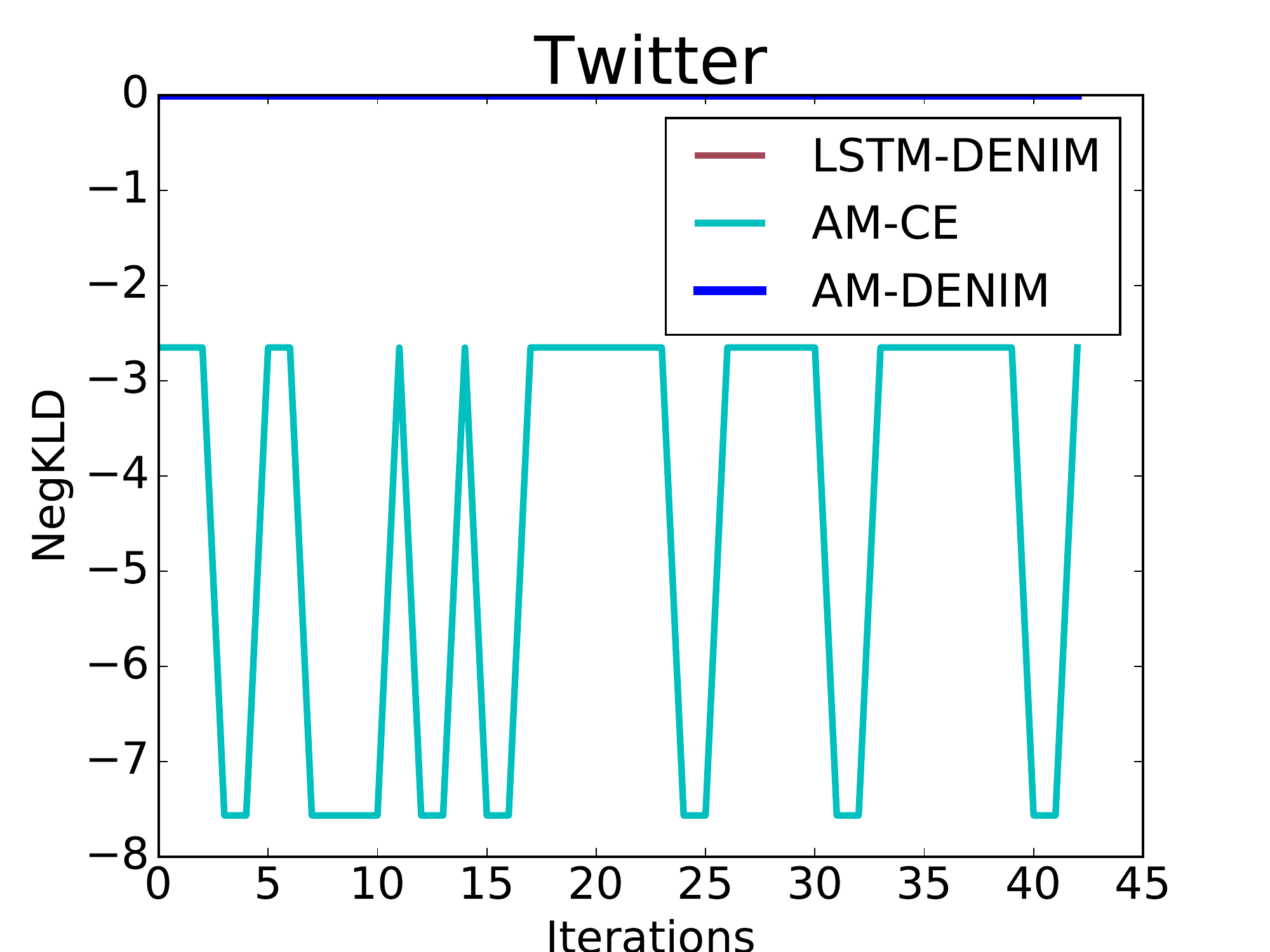}
\label{fig:conv}}
\hspace{15pt}
\subfigure[Change in Quantification performance with distribution drift.]{
  \includegraphics[height=1.2in]{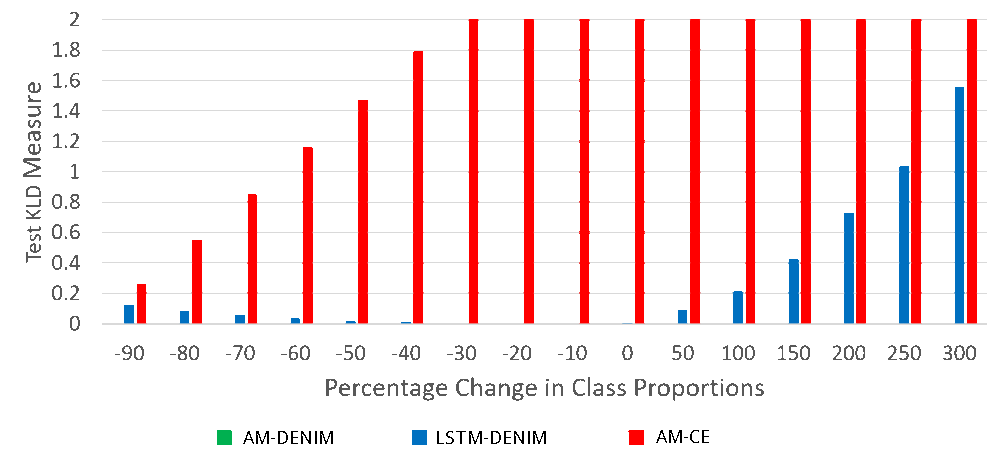}
  \label{fig:pvalLSTM}}
\caption{Results on the Twitter Sentiment Analysis Task}
\label{fig:twitter_plot}
\end{figure}

We also experiment with artificially changing the fraction of positive and negative examples in order to see the performance of our model under distribution drift (see Figure~\ref{fig:pvalLSTM}). The fraction of negatives and positives in the test set was distorted from their original values by resampling. As the test distribution priors are distorted more and more, \textbf{AM-CE} (Attention Model trained with Cross Entropy) performs extremely poorly. \dnemsis with LSTMs displays some degree of robustness to drift but succumbs at extremely high level of drift. \dnemsis with AM models on the other hand, remains extremely robust to even high degree of distribution drift, offering near-zero \kld error.

The benefits of the attention models employed by \dnemsis allow it to identify critical words in a tweet that clearly signal its polarity. The highlighted words (see Figure~\ref{fig:attention}) are those for which \dnemsis assigned an attention score $\alpha \approx 1$.

\begin{figure}[h]%
\centering
\tiny{
\begin{tabular}{l}
\hline
TGIF!! Make it a \colorbox{green!30}{great} day, Robbie!!\\\hline
Monsanto's Roundup \colorbox{red!30}{not} good for you\\\hline
I may be in \colorbox{green!30}{love} with Snoop\\\hline
anyone having problems with Windows 10? may be coincidental\\
\colorbox{red!30}{but} since i downloaded, my WiFi keeps dropping out.\\\hline
@NariahCFC \colorbox{red!30}{against} barca pre season stand out player 1st half.. \\\hline
@alias8818 Hey \colorbox{green!30}{there!} We're excited to have you as part of the\\
 T-Mobile family!\\\hline
listening to Fleetwood Mac and having my candles lit is the\\
 \colorbox{green!30}{perfect} Sunday evening\\\hline
\end{tabular}
}
\caption{Figuring where the attention is. Highlighted words got high attention scores. A red (green) highlight indicates that the tweet was tagged with a negative (positive) sentiment.}%
\label{fig:attention}%
\end{figure}




\bibliographystyle{spmpsci}      
\bibliography{Fabrizio,refs-icml15-tale-of-two-classes,newrefs}

\appendix

\normalsize

\section{Proof of Theorem~\ref{thm:dspade-conv}}
\label{app:dspade-proof}
\begin{theorem}
Consider a concave performance measure defined using a link function $\Psi$ that is concave and $L'$-smooth. Then, if executed with a uniform step length satisfying $\eta < \frac{2}{L}$, then \dspade $\epsilon$-stabilizes within $\softO{\frac{1}{\epsilon^2}}$ iterations. More specifically, within $T$ iterations, \dspade identifies a model $\vw^t$ such that $\norm{\nabla^t}_2 \leq \bigO{\sqrt{L'\frac{\log T}{T}}}$.
\end{theorem}
\begin{proof}
Recall that we assume that the reward functions $r(f(\vx;\vw),y)$ are $L$-smooth functions of the model $\vw$. This is satisfied by all reward functions we consider. Note, however, that nowhere will we assume that the reward functions are concave in the model parameters. We will use the shorthand $\nabla^t = \nabla_\vw g(\vw^t; S_t, \alpha^t, \beta^t)$ and $F(\vw^t,\valpha^t) = g(\vw^t; S_t, \alpha^t, \beta^t)$. We will prove this result for the batch version of the \dspade algorithm for the sake of simplicity and to present the key ideas. The extension to the mini-batch version is straightforward and will introduce an additional error of the order of $\frac{1}{\sqrt b}$ where $b$ is the batch size.

The batch version of \dspade makes the following model update $\vw^{t+1} = \vw^t + \eta\cdot\nabla^t$. Using the smoothness of the reward functions, we get
\[
F(\vw^{t+1},\valpha^t) \geq F(\vw^t,\valpha^t) + \ip{\nabla^t}{\vw^{t+1}-\vw^t} - \frac{L}{2}\norm{\vw^{t+1}-\vw^t}_2^2,
\]
which, upon rearranging, give us $\norm{\nabla^t}_2^2 \leq \frac{F(\vw^{t+1},\valpha^t) - F(\vw^t,\valpha^t)}{\eta\br{1 - \frac{\eta L}{2}}}$, which, upon summing up, gives us
{\[
\sum_{t=1}^T \norm{\nabla^t}_2^2 \leq \frac{1}{\eta\br{1 - \frac{\eta L}{2}}}\br{F(\vw^{T+1},\valpha^T) + \sum_{t=2}^T F(\vw^t,\valpha^{t-1}) - F(\vw^t,\valpha^t)}.
\]}%
However, by a forward regret-analysis of the dual updates which execute the follow-the-leader algorithm and the fact that due to the $L'$-smoothness of $\Psi$, the functions $F(\vw,\valpha)$ are $\frac{1}{L'}$ strongly convex,
\[
\sum_{t=2}^T F(\vw^t,\valpha^{t-1}) - F(\vw^t,\valpha^t) \leq \bigO{L'\log T}.
\]
This completes the proof upon applying an averaging argument.
\end{proof}


\section{\damp: A Deep Learning Technique for Pseudolinear Performance Measures}
\label{app:damp}
\begin{figure*}[t]
\centering
\subfigure[CT]{
\includegraphics[scale=0.25]{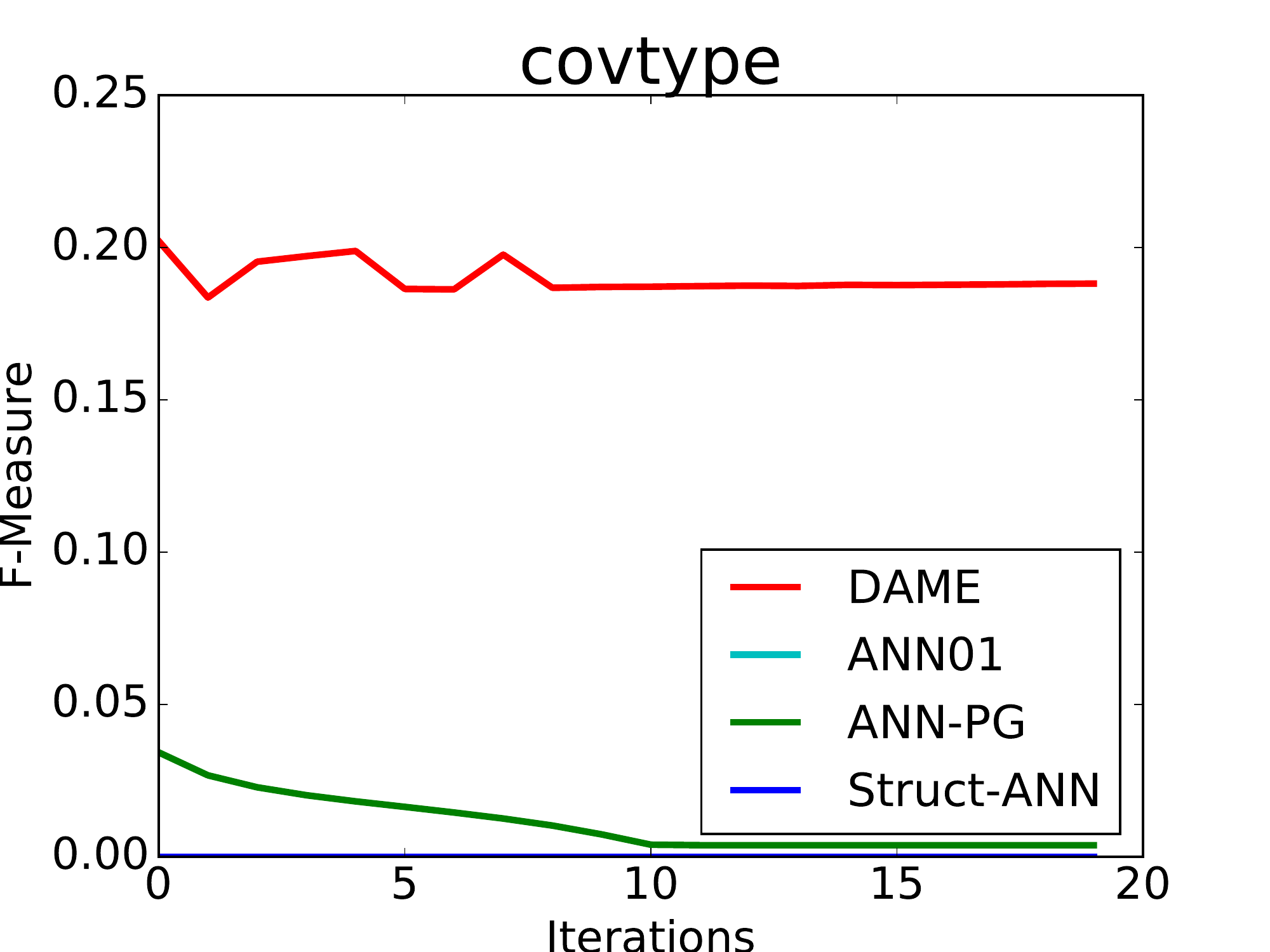}
\label{subfig:f1-cov}
}\hspace*{-5pt}
\subfigure[IJCNN]{
\includegraphics[scale=0.25]{kdd08_F1.pdf}
\label{subfig:f1-kdd}
}\hspace*{-5pt}
\subfigure[IJCNN1]{
\includegraphics[scale=0.25]{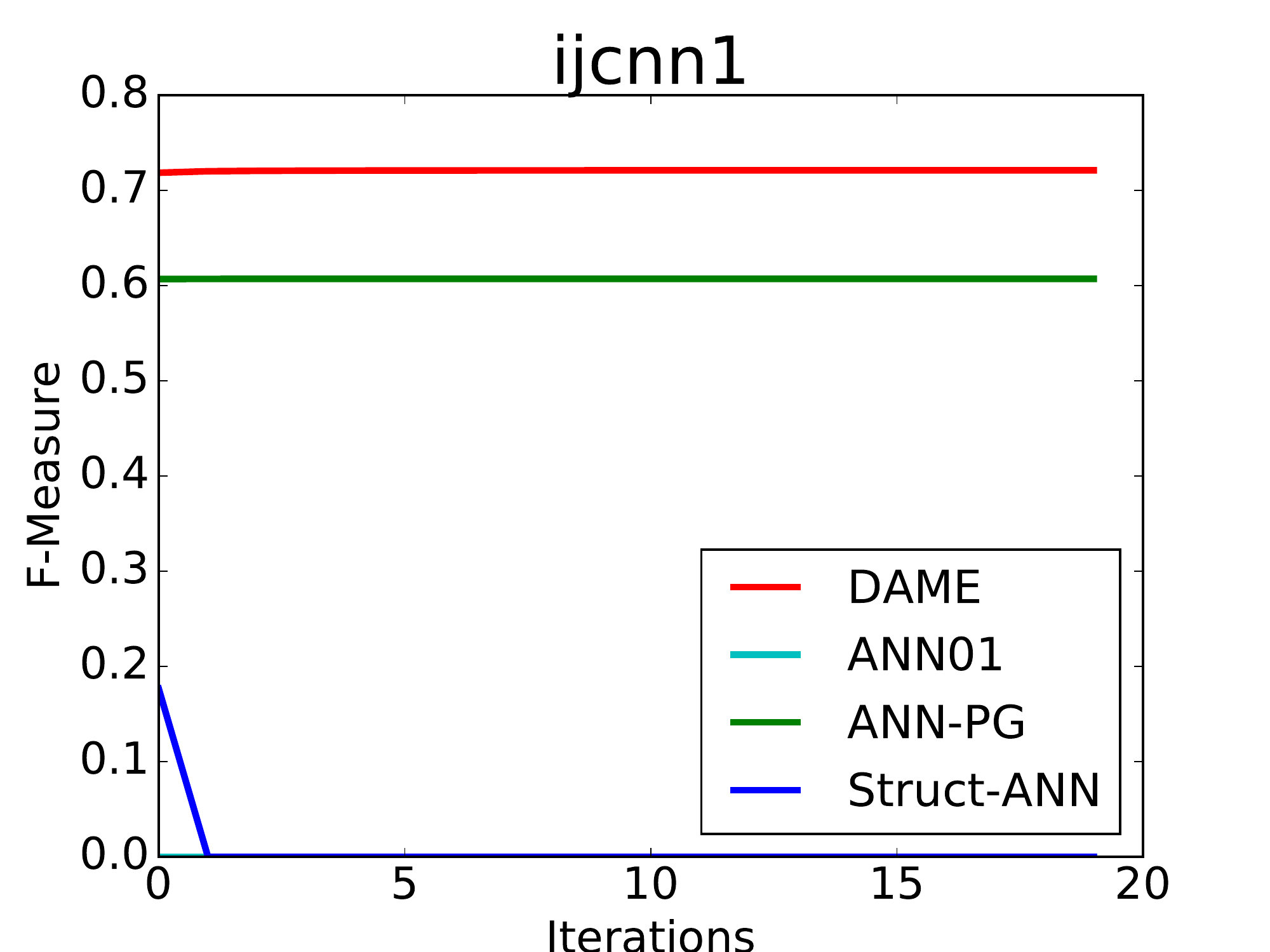}
\label{subfig:f1-ijcnn}
}\\
\hspace*{-5pt}
\subfigure[IJCNN1]{
\includegraphics[scale=0.25]{cod-rna_F1.pdf}
\label{subfig:f1-cod}
}\hspace*{-5pt}\subfigure[IJCNN1]{
\includegraphics[scale=0.25]{letter_F1.pdf}
\label{subfig:f1-letter}
}\hspace*{-5pt}
\subfigure[IJCNN1]{
\includegraphics[scale=0.25]{a9a_F1.pdf}
\label{subfig:f1-a9a}
}
\caption{Experiments with DAME}
\label{fig:F1}
\end{figure*}

We present an algorithm  for training deep models on pseudolinear performance measures such as F-measure. These are extremely popular in several areas and direct optimization routines are sought after. We know of only one proposed algorithm for training deep models with F-measure in the work of \cite{Eban17}. However, their algorithm involves constrained optimization routines over deep models and the authors do not discuss the details of implementing the same.

Our algorithm \damp, on the other hand, is based on an alternating strategy, is very scalable and gives superior performance across tasks and datasets. For sake of simplicity, we represent the pseudolinear performance measure as
\[
\cP_{(\va,\vb)}(\vw) = \frac{\cP_{\va}(\vw)}{\cP_{\vb}(\vw)} = \frac{a_0 + a_1\cdot \tpr(\vw) + a_2\cdot \tnr(\vw)}{b_0 + b_1\cdot \tpr(\vw) + b_2\cdot \tnr(\vw)}
\]
Given the above, we define the notion of a valuation function.
\begin{definition}[Valuation Function]
The valuation of a pseudolinear measure $\cP_{(\va,\vb)}(\vw)$ at any level $v > 0$, is defined as
\[
V(\vw,v) = \cP_{\va}(\vw) - v\cdot\cP_{\vb}(\vw)
\]
\end{definition}

We will use the notation $\cP_{(\va,\vb),S}(\vw^t)$ and $V_S(\vw,v)$ to denote respectively, the performance measure, and the valuation function as defined on a data sample $S$. At every time step $t$, \damp looks at $v^t = \cP_{(\va,\vb)}(\vw^t)$ and attempts to approximate the task of optimizing F-measure (or any other pseudolinear measure) using a cost weighted classification problem described by the valuation function at level $v^t$. After making updates to the model with respect to this approximation, \damp refines the approximation again, and so on.

We note that similar alternating strategies have been studied in literature in the context of F-measure before \cite{KoyejoNRD14,NarasimhanKJ2015} and offer provable convergence guarantees for linear models. However, a direct implementation of these methods gives extremely poor results as we shall see in the next section. The complex nature of these performance measures, that are neither convex nor concave, make it more challenging to train deep models.

To solve this problem, \damp utilizes a two-stage training procedure, involving pretraining the entire network (i.e. both upper and lower layers) on a standard training objective such as cross-entropy or least squares, followed by fine tuning of \emph{only the upper layers} of the network to optimize F-measure. The pretraining is done using standard stochastic mini-batch gradient descent.

The details of the algorithm are given in Algorithm~\ref{algo:damp}. For sake of simplicity we will let $(\vw_1,\vw^2)$ denote a stacking of the neural networks described by the models $\vw_1$ and $\vw_2$. More specifically $\vw_2$ denotes a network with input dimensionality $d_\text{in}$ and output dimensionality $d_\text{int}$ whereas $\vw_1$ denotes a network with input dimensionality $d_\text{int}$ and output dimensionality $d_\text{out}$. To ensure differentiability, \damp uses valuation functions with appropriate reward functions replacing the \tpr and \tnr functions.

We are able to show stronger \emph{local convergence} guarantees for \damp. Due to lack of space, we present only a sketch of the proof for the batch version i.e. $S_{t,i} = \tilde T$ for all time steps $t,i$, with constant step lengths. We will continue to assume that the valuation functions are $L$-smooth functions of the upper model. It is also noteworthy that we present the guarantee only for the fine-tuning phase since the pre-training phase enjoys local convergence guarantees by standard arguments. For this reason, we will omit the lower network in the analysis.

We will assume that the performance measure satisfies $\cP_{\va}(\vw) \leq M$ for all $\vw \in \cW$ and $\cdot\cP_{\vb}(\vw) \geq m$ for all $\vw \in cW$. We note that these assumptions are standard \cite{KarLNCS2016,NarasimhanKJ2015} and also readily satisfied by F-measure, Jaccard coefficient etc for which we have $m, M = \Theta(1)$ (see \cite{NarasimhanKJ2015}). Let $\kappa = 1 + M/m$.

To prove Theorem~\ref{thm:damp-conv}, we first show the following result. Since we have $\nabla_{\vw}\cP_{(\va,\vb)}(\vw) = \frac{\nabla_{\vw}V(\vw,\cP_{(\va,\vb)})}{\cP_{\vb}(\vw)}$, and $\cP_{\vb}(\vw) \geq m$, Theorem~\ref{thm:damp-conv} will follow

\begin{theorem}
If executed with a uniform step length satisfying $\eta < \frac{2}{L\kappa}$, then \damp discovers an $\epsilon$-stable model within $\bigO{\frac{1}{\epsilon^2}}$ inner iterations. More specifically, within $\frac{\kappa^2}{m}\frac{1}{\eta\br{1 - \frac{L\kappa\eta}{2}}\epsilon^2}$ iterations, \damp identifies a model $\vw^t_1$ such that $\norm{\nabla_{\vw^{t}_1} V_{\tilde S^t}((\vw^{t}_1,\vw^{t}_2),v^{t-1})}_2 \leq \epsilon$.
\end{theorem}
\begin{proof}
It is easy to see that $V(\vw^{t-1}_1,v^t) = 0$ and that $V(\vw_1,v)$ is a $L\kappa$-smooth function of the model parameter $\vw_1$ for any realizable valuation i.e. $v = \cP_{(\va,\vb)}(\vw)$ for some $\vw \in \cW$. Now, the batch version of the \damp algorithm makes the following model updates within the inner loop
\[
\vw^{t-1,t'}_1 = \vw^{t-1,t'-1}_1 + \eta\cdot\nabla^{(t-1,t')},
\]
where $\nabla^{(t-1,t')} = \nabla_{\vw^{t-1,t'-1}_1} V(\vw^{t-1,t'-1}_1,v^t)$. Using the smoothness of the reward functions, we get
\begin{align*}
V(\vw^{t-1,t'}_1,v^t) &\geq V(\vw^{t-1,t'-1}_1,v^t) + \ip{\nabla^{(t-1,t')}}{\vw^{t-1,t'}_1\vw^{t-1,t'-1}_1}\\
 &\qquad - \frac{L\kappa}{2}\norm{\vw^{t-1,t'}_1\vw^{t-1,t'-1}_1}_2^2\\
													&= V(\vw^{t-1,t'-1}_1,v^t) + \eta\br{1 - \frac{L\kappa\eta}{2}}\norm{\nabla^{(t-1,t')}}_2^2
\end{align*}
Now this shows that at each step where $\norm{\nabla^{(t-1,t')}} > \epsilon$, the valuation of the model $\theta^{(t+1,i)}$ goes up by at least $\eta\br{1 - \frac{L\kappa\eta}{2}}\epsilon^2$. It is easy to see that if $V(\vw^t_1,v^t) \geq c$ then $\cP(\vw^t_1) \geq \cP(\vw^{t-1}_1) + \frac{c}{M}$. Since the maximum value of the performance measure for any model is $\frac{M}{m}$, putting these results together tell us that \damp cannot execute more than $\frac{M^2}{m}\frac{1}{\eta\br{1 - \frac{L\kappa\eta}{2}}\epsilon^2}$ inner iterations without encountering a model $\vw^{t,t'}_1$ such that $\norm{\nabla^{(t-1,t')}}_2 \leq \epsilon$. An easy calculation shows that for such a model we also have $\norm{\nabla_\vw \cP(\vw^{t,t'})}_2 \leq \frac{\epsilon}{m}$ as well.
\end{proof}

The following experiments Figure:[\ref{fig:F1}] show the performance of DAME on the F1-Measure. A naive training with misclassification loss yields extremely poor F-measure performance. Moreover, a naive implementation of methods proposed for linear models such as the plug-in method also performs very poorly. DAME on the other hand is able to rapidly offer very good F-measure scores after looking at a fraction of the total data.

As, it can be seen, struct ANN provides a consistent poor performance whereas it seems to perform well in the linear case. This is because our implementation of the \textbf{Struct ANN} is a minibatch method and the gradient obtained from the structual method has almost no real information due to this. In other variants of the application\cite{song2016} of the structual ANN, people have usually used full batch methods. We would like to point out that such a case is almost intractable with respect to memory and computation time.

\section{Details of implementation of the Structual ANN from \cite{song2016}}
\label{sec:structual-ann}

Here assume that $\Delta$ is the loss function we are looking at and its input is the two dimensional confusion matrix. Keeping this is mind, we define the following functions.
\[
  a(\hat\vy, \vy) = \sum_i \mathcal{I}\{\vy_i=1\} \mathcal{I}\{\hat\vy_i=1\}
\]
\[b(\hat\vy, \vy) =  \sum_i \mathcal{I}\{\vy_i=1\} \mathcal{I}\{\hat\vy_i=0\}\]
\[c(\hat\vy, \vy) = \sum_i \mathcal{I}\{\vy_i=0\} \mathcal{I}\{\hat\vy_i=1\}\]
\[d(\hat\vy, \vy) = \sum_i \mathcal{I}\{\vy_i=0\} \mathcal{I}\{\hat\vy_i=0\}\]
Finally, $m(\cdot)$ is the artificial neural network
\[
f(\vw) = \max_{\hat\vy}\bc{\Delta\br{a(\hat\vy,\vy),b(\hat\vy,\vy),c(\hat\vy,\vy),d(\hat\vy,\vy)} + \sum_{i=1}^n (\hat y_i - y_i)m(\vx_i) }
\]
\[
\min_\vw g(\vw) = \partial m(\vx_i) + C\cdot f(\vw)
\]
\[
\partial g(\vw) \ni \vw + \partial f(\vw)
\]
If
\[
\tilde\vy \in \underset{\hat\vy}{\arg\max}\bc{\Delta\br{a(\hat\vy,\vy),b(\hat\vy,\vy),c(\hat\vy,\vy),d(\hat\vy,\vy)} + \sum_{i=1}^n (\hat y_i - y_i)m(\vx_i)},
\]
then
\[
\sum_{i=1}^n (\tilde y_i - y_i)\partial m(\vx_i)\in \partial f(\vw)
\]

Hence, to find $\tilde y$, we need to solve the following
\[
\underset{(p,q,r,s)}{\arg\max}\underset{\underset{a(\hat\vy,\vy)=p,b(\hat\vy,\vy)=q,c(\hat\vy,\vy)=r,d(\hat\vy,\vy)=s}{\hat\vy \text { such that }}}{\arg\max}\bc{\Delta\br{p,q,r,s} + \sum_{i=1}^n (\hat y_i - y_i)m(x_i)},
\]
\[
\underset{(p,q,r,s)}{\arg\max}\bc{\Delta\br{p,q,r,s} + \underset{\underset{a(\hat\vy,\vy)=p,b(\hat\vy,\vy)=q,c(\hat\vy,\vy)=r,d(\hat\vy,\vy)=s}{\hat\vy \text { such that }}}{\arg\max}\bc{\sum_{i=1}^n (\hat y_i - y_i)s_i}},
\]
\[
\underset{(p,q,r,s)}{\arg\max}\bc{\Delta\br{p,q,r,s} + \underset{\underset{a(\hat\vy,\vy)=p,b(\hat\vy,\vy)=q,c(\hat\vy,\vy)=r,d(\hat\vy,\vy)=s}{\hat\vy \text { such that }}}{\arg\max}\bc{\sum_{i=1}^n \hat y_is_i}},
\]

This is very amiable to a symbolic gradient operation, as we need to find the gradient which looks like
\[
\sum_{i=1}^n (\tilde y_i - y_i)\partial m(\vx_i)
\]

However, by linearity, this is the same as
\[
\partial \sum_{i=1}^n (\tilde y_i - y_i) m(\vx_i)
\]

Therefore , we need to do a forward pass over the symbolic graph to get the value of $m(\vx_i)$ and then feed to our solver for the most violated constraint, which will give us $\tilde y_i$ and then we define the symbolic gradient as
\[
\partial \sum_{i=1}^n (\tilde y_i - y_i) m(\vx_i)
\]

\end{document}